\def\year{2021}\relax
\pdfoutput=1
\documentclass[letterpaper]{article} 
\usepackage{aaai21}  
\usepackage{times}  
\usepackage{helvet} 
\usepackage{courier}  
\usepackage[hyphens]{url}  
\usepackage{graphicx} 
\usepackage{natbib}  
\usepackage{caption} 
\usepackage[hidelinks]{hyperref}
\RequirePackage{amsmath}
\RequirePackage{amsthm}
\RequirePackage{amssymb}

\newcommand\relphantom[1]{\mathrel{\phantom{#1}}}

\usepackage{paralist}
\usepackage{cuted}

\usepackage[linesnumbered]{algorithm2e}

\usepackage[small]{complexity}
\newclass{\sharpP}{\text{\rm\#}P}

\newtheorem{theorem}{Theorem}

\newtheorem{lemma}{Lemma}
\newtheorem{corollary}{Corollary}
\newtheorem{definition}{Definition}

\newtheorem{claim}{Claim}
\newtheorem{claim theorem}{Claim of Theorem}
\newtheorem{claim proposition}{Claim of Proposition}
\newtheorem{claim lemma}{Claim of Lemma}
\newtheorem{claim corollary}{Claim of Corollary}
\def\newsymbol{\raise.75pt\hbox{\llap{\large$\star$\hspace{0.5ex}}}}
\newtheorem{new proposition}{\newsymbol Proposition}
\newtheorem{new lemma}{\newsymbol Lemma}
\newtheorem{new corollary}{\newsymbol Corollary}

\newtheorem{proposition}{Proposition}

\theoremstyle{definition}
\newtheorem{example}{Example}

\def\exampleqed{\hfill$\diamond$}

\def\bouquet{\mathcal{B}}
\def\cliques{\Pi}
\def\separators{\Delta}

\def\chordalcomps{\mathcal{C}}

\def\balap#1{\hbox to 0pt{\scriptsize$#1$}}

\def\amo{\text{\rm AMO}}
\def\hamo{\text{\rm\#\kern-0.5ptAMO}}
\def\spe{\text{\rm Pe}}
\def\pe{\text{\rm\#\kern0.1ptPe}}

\def\cC{\mathcal C}

\usepackage{tikz}
\usetikzlibrary{graphs, arrows.meta, backgrounds, calc, decorations.pathmorphing, shapes.misc}

\definecolor{ba.yellow}{RGB}{252,190,18}
\definecolor{ba.gray}{RGB}{153,153,156}
\definecolor{ba.blue}{RGB}{6,123,164}
\definecolor{ba.red}{RGB}{213,96,98}
\definecolor{ba.orange}{RGB}{233,116,81}
\definecolor{ba.pine}{RGB}{67,154,134}
\definecolor{ba.green}{RGB}{196,247,161}
\definecolor{ba.violet}{RGB}{88, 53, 94}

\tikzset{
  pico/.style = {
    every node/.style = {
      draw,
      circle,
      semithick,
      inner sep = 0pt,
      minimum width = 0.7ex,
      fill = white
    },
    semithick
  },
  edge/.style = {
    semithick
  },
  arc/.style = {
    edge,
    ->,
    >={[round,sep]Stealth}
  },
}
\tikzset{
  axis/.style = {
    semithick,
    ->,
    >={[round,sep]Stealth},
  },
  tick/.style = {
    thin,
    font=\small
  },
  timeout/.style = {
    semithick,
    densely dashed,
    color = ba.red,
    font=\small,
  },
  mean_dot/.style = {
    draw, fill,
    circle,
    inner sep = 0pt,
    minimum width = 1mm
  }
}

\def\rhombus{%
  \node (1) at (-1,0) {$1$};%
  \node (2) at (0,1)  {$2$};%
  \node (3) at (0,-1) {$3$};%
  \node (4) at (1,0)  {$4$};%
}
\def\marcelgraph{%
  \node (1) at (0,0)  {$1$};
  \node (2) at (1,0)  {$2$};
  \node (3) at (2,0)  {$3$};
  \node (4) at (0,-1) {$4$};
  \node (5) at (1,-1) {$5$};
  \node (6) at (2,-1) {$6$};
  \graph[use existing nodes, edges = {edge}] {
    1 -- 2 -- 3 --[bend right] 1;
    4 -- 5 -- 6;
    2 -- {4,5,6};
    3 -- {4,5,6};
  };
}

\tikzset{
  axis/.style = {
    semithick,
    ->,
    >={[round,sep]Stealth},
  },
  tick/.style = {
    thin,
    font=\small
  },
  timeout/.style = {
    semithick,
    densely dashed,
    color = ba.red,
    font=\small,
  },
  mean_dot/.style = {
    draw, fill,
    circle,
    inner sep = 0pt,
    minimum width = 1mm
  }
}

\tikzset{
  cactus/.style = {
    draw,
    semithick,
    circle,
    inner sep = 0.03cm
  },
  cdfp-cp/.style = {
    cactus,
    color = ba.pine,
    fill  = ba.pine!50
  },
  cdfp-dp/.style = {
    cactus,
    color = ba.violet,
    fill  = ba.violet!50
  },
  cdfp-tw/.style = {
    cactus,
    color = ba.yellow,
    fill  = ba.yellow!50
  },
  cdfp-li/.style = {
    cactus,
    color = ba.blue,
    fill  = ba.blue!50
  },
}

\title{Polynomial-Time Algorithms for Counting and Sampling \\ Markov Equivalent DAGs\thanks{Extended version of paper accepted to the Proceedings of the 35th AAAI Conference on Artificial Intelligence (AAAI-2021).}}

\author {
     Marcel Wienöbst, Max Bannach, Maciej Li\'{s}kiewicz \\
}

 \affiliations{
     Institute of Theoretical Computer Science, University of L\"{u}beck, Germany \\
     \{wienoebst,bannach,liskiewi\}@tcs.uni-luebeck.de
 }

\begin{document}
\maketitle\vspace*{-5.01mm}

\begin{abstract}
Counting and sampling directed acyclic graphs
from a Markov equivalence class are fundamental tasks in graphical causal 
analysis. In this paper, we show that these tasks can be performed in 
polynomial time, solving a long-standing open problem in this 
area. Our algorithms are effective and easily implementable. 
Experimental results show that the algorithms significantly 
outperform state-of-the-art methods.
\end{abstract}

\section{Introduction}
Graphical modeling plays a key role in causal theory, allowing to
express complex causal phenomena in an elegant, mathematically sound way.
One of the most popular graphical models are directed acyclic graphs
(DAGs), which represent direct causal influences
between random variables by directed
edges~\cite{spirtes2000causation,pearl2009causality,koller2009probabilistic}.
They are commonly used in empirical sciences to discover and understand 
causal effects. 
However, in practice, the underlying DAG is usually unknown,
since, typically, no single DAG explains the observational data.
Instead, the statistical properties of the data are maintained by a
number of different DAGs, which constitute a Markov equivalence class
(MEC, for short). Therefore, these DAGs are indistinguishable on the
basis of observations alone
\cite{verma1990equivalence,verma1992algorithm,heckerman1995learning}.

It is of great importance to investigate model learning and to analyze causal 
phenomena using MECs directly rather than the DAGs themselves. 
Consequently, this has led to intensive studies on Markov equivalence
classes of DAGs and resulted in a long and successful track record. 
Our work contributes to this line of research  by providing 
the first polynomial-time algorithms for \emph{counting} and 
for \emph{uniform sampling} Markov 
equivalent DAGs \textendash\ two important primitives in both 
theoretical and experimental studies. 

Finding the graphical criterion for two DAGs to be Markov equivalent
\cite{verma1990equivalence} and providing the graph-theoretic 
characterization of MECs as so-called CPDAGs
\cite{Andersson1997} mark key turning  points in this research direction. 
In particular, they have contributed to the progress of computational 
methods in this area. Important advantages of the modeling with CPDAGs
are demonstrated by algorithms that learn causal structures from observational data
\cite{verma1992algorithm,Meek1995,meek1997graphical,spirtes2000causation,chickering2002learning,chickering2002optimal};
and that analyze causality based on a given MEC, rather than a single DAG
\cite{maathuis2009estimating,van2016separators,perkovic2017complete}.
Algorithms that ignore Markov equivalence may lead to incorrect solutions. 

A key characteristic of an MEC 
is its size, i.\,e., 
the number of DAGs in the class. It indicates uncertainty 
of the causal model inferred from observational data and
it serves as an indicator for the performance of recovering true causal effects.
Moreover, the feasibility of causal inference methods is often highly
dependent on the size of the MEC; e.\,g., to estimate the average
causal effects from observational data for a given CPDAG, as proposed
by \citet{maathuis2009estimating}, one has to consider all DAGs in the class.
Furthermore, computing the size of a Markov equivalence class is
commonly used as a subroutine in practical algorithms. For example, when
actively designing interventions, in order to identify the underlying
true DAG in a given MEC, the size of the Markov equivalence subclass
is an important metric to select the best intervention target
\cite{He2008,hauser2012characterization,shanmugam2015learning,ghassami2018budgeted,Ghassami2019}. 

The first algorithmic approaches for counting the number of Markov equivalent
DAGs relied on exhaustive search \cite{Meek1995,madigan1996bayesian} 
based on the graphical characterization of \citet{verma1990equivalence}. 
The methods are computationally expensive as the size of an MEC represented 
by a CPDAG may be superexponential in the number of vertices of the graph.
More recently, \citet{He2015} proposed
to partition the MEC by fixing root variables in any undirected component of the CPDAG.
This yields a recursive strategy for counting Markov equivalent DAGs, which forms the basis 
of several ``root-picking'' algorithms
\cite{He2015,Ghassami2019,Ganian2020}. 
As an alternative approach, recent methods utilize
dynamic programming on the clique tree representation of chordal 
graphs and techniques from intervention design \cite{Talvitie2019,Teshnizi20}.

The main drawback of the existing counting algorithms 
is that they have exponential worst-case run time.
Moreover, as our experiments show, the 
state-of-the-art algorithms 
\cite{Talvitie2019,Ganian2020,Teshnizi20}
perform inadequately in practice on a wide range of instances.

The main achievement of our paper is the 
first poly\-no\-mial-time algorithm for counting and
for sampling Markov equivalent DAGs. The counting algorithm,
called \emph{Clique-Picking}, explores the clique tree representation 
of a chordal graph, but it avoids the use of computationally intractable 
dynamic programming on the clique tree. 
The Clique-Picking algorithm is effective,  easy to implement, and
our experimental results show that it significantly 
outperforms the state-of-the-art methods.
Moreover, we show that, using the algorithm in a preprocessing phase,
uniform sampling of Markov equivalent DAGs can be performed 
in linear time.

We prove that our  results are tight
in the sense that counting Markov equivalent
DAGs that encode additional background knowledge
is intractable under standard complexity-theoretic
assumptions. This justifies the exponential time approaches 
by~\citet{Meek1995} and \citet{Ghassami2019}.

The next section contains preliminaries on graphs and MECs.  In Sec.~\ref{sec:lbfsmaos}, we present the ideas of our
novel approach, 
and Sec.~\ref{sec:separators} explains how to avoid overcounting using
minimal separators. 
Section~\ref{sec:cliquepicking} contains our algorithm and in
Sec.~\ref{sec:compl} we analyze its time complexity and formally present
the main results of the paper.  Finally, Sec.~\ref{sec:exp}
shows our experimental results.  Due to space constraints, proofs are
relocated to the appendix. We provide short proof sketches for the
most important results in the main text.

\section{Preliminaries}\label{sec:preliminaries}

A graph $G = (V_G, E_G)$ consists of a set of vertices $V_G$ and a set
of edges $E_G \subseteq V_G \times V_G$. Throughout the paper,
whenever the graph $G$ is clear from the context, we will drop the
subscript in this and analogous notations.
An edge $u-v$ is undirected
if $(u, v), (v, u) \in E_G$ and directed $u \rightarrow v$ if
$(u,v) \in E_G$ and $(v,u) \not\in E_G$. Graphs which contain
undirected and directed edges are called partially directed. Directed
acyclic graphs (DAGs) 
contain only directed
edges and no directed cycle. 
We refer to the 
neighbors of a vertex~$u$ in $G$ as $N_G(u)$
and denote the induced subgraph of $G$ on a set $C\subseteq V$ by $G[C]$.  
The graph union $G \cup H$ includes edges present in $G$ or in $H$\footnote{For example,
the union of  $a \to b \to c$ and $a \gets b \to c$ is the graph
$a - b \to c$.}. 

The \emph{skeleton} of a partially directed graph $G$ is the undirected 
graph that results from ignoring edge directions.
A \emph{v-structure} in a partially directed 
graph $G$ is an ordered triple of vertices $(a,b,c)$ which
induce the subgraph $a \rightarrow b \leftarrow c$.


A \emph{clique} is a set $K$ of pairwise adjacent vertices. We denote
the set of all maximal cliques of $G$ by $\cliques(G)$.  In a
connected graph, we call a set $S\subseteq V$ an
\emph{$a$-$b$-separator} for two nonadjacent vertices $a,b\in V$ if
$a$ and $b$ are in different connected components in
$G[V\setminus S]$. If no proper subset of $S$ separates $a$ and $b$ we
call $S$ a \emph{minimal $a$-$b$-separator}. We say a set
$S$ is a \emph{minimal separator} if it is a minimal $a$-$b$-separator
for any two vertices\footnote{Observe that a minimal separator can be a proper
subset of another minimal separator (for different vertex pairs $a$-$b$).}. We denote the set of all minimal
separators of a graph $G$ by $\separators(G)$.
An undirected graph is called
\emph{chordal} if no subset of four or more vertices induces an
undirected cycle. For every chordal graph on $n$ vertices we have
$|\cliques(G)|\leq n$~\cite{dirac61}. Furthermore, it is well-known
that a graph $G$ is chordal if, and only if, all its minimal
separators are cliques.


%
A Markov equivalence class (MEC) consists of DAGs
encoding the same set of conditional independence relations among the variables. 
Due to \citet{verma1990equivalence}, we know that 
two DAGs are Markov equivalent if, and only if,
they have 
the same skeleton and the same v-structures.
An MEC can be represented by a CPDAG (\emph{completed partially directed acyclic graph}), 
which is the union graph of the DAGs in the equivalence class it represents. 
The undirected components of a CPDAG are
\emph{undirected and connected chordal graphs} (UCCGs)~\cite{Andersson1997}.

An orientation of a partially directed graph $G$ is obtained by replacing 
each undirected edge with a directed one. Such an orientation is called 
\emph{acyclic} if it does not contain a directed cycle and 
\emph{moral} if it does not create a new v-structure (sometimes called immorality).
In the following, we will only consider acyclic moral orientations 
(AMOs). For a partially ordered graph $G$,
we denote by $\amo(G)$ the set of all AMOs and by 
$\hamo(G)$ 
the number of AMOs of $G$.
In particular, if $G$ is a CPDAG representing an MEC then 
$\hamo(G)$ is the size of the class.
In this paper, we also refer to the \emph{computational problem}
of counting the number of AMOs for a given CPDAG as $\hamo$.

In case we have an induced subgraph $a
\rightarrow b - c$ in a partially directed graph, the edge between $b$
and $c$ is oriented $b \rightarrow c$ in all AMOs. This is known
as the first Meek rule~\cite{Meek1995}.

For a CPDAG $G$, the AMOs of each UCCG of $G$ 
can be chosen independently of the other UCCGs 
and the directed part of $G$~\cite{Andersson1997}.
Thus,   
\begin{equation*}\label{eq:hamo:in:cpdag}
  \hamo(G) =\prod_{\text{$H$ is UCCG in $G$}} \hamo(H) .
\end{equation*}
Hence, the problem $\hamo$ of counting the number of DAGs in an MEC 
reduces to counting the number of AMOs in a UCCG~\cite{Gillispie2002,He2008}.

An AMO $\alpha$ of a graph $G$ can be represented by a (not necessarily unique)
linear ordering of the vertices. Such a
\emph{topological} ordering $\tau$ represents $\alpha$ if for each edge $u
\rightarrow v$ in $\alpha$, $u$ precedes $v$ in $\tau$. 
We denote all topological orderings representing an AMO $\alpha$ of a graph~$G$ by
$\mathrm{top}_G(\alpha) = \{\tau_1, \dots, \tau_\ell\}$.
Note that every AMO of a UCCG contains exactly
one source vertex, i.\,e., a vertex with no incoming edges.

The $s$-orientation $G^s$ of a UCCG $G$ is the union of all AMOs
of $G$ with unique source vertex~$s$. 
We view $s$-orientations from the equivalent
perspective of being the union of all AMOs that can be represented by
a topological ordering starting with~$s$. The undirected components of~$G^s$
are UCCGs and can be oriented independently~\cite{He2015}.
This observation enables recursive strategies for counting 
AMOs: the ``root-picking''
approaches~\cite{He2015,Ghassami2019,Talvitie2019,Ganian2020} that
pick each vertex~$s$ as source and recurse on
the UCCGs of the $s$-orientation.
Because these UCCGs can be oriented independently, the number of AMOs
is obtained by alternately summing over
the number of AMOs for each source vertex $s$ and multiplying the number of
AMOs for each independent UCCG.

\section{Lexicographic BFS and AMOs}
\label{sec:lbfsmaos}
We introduce the core ideas of our algorithm
for \hamo\ and a linear-time algorithm for finding the UCCGs of
the $s$-orientations and their generalization, the
\emph{$\pi(K)$-orientations}. We do this by connecting AMOs with
so-called \emph{perfect elimination orderings} (PEOs).
A linear ordering of the vertices is a PEO if for each vertex $v$, the neighbors of $v$ that
occur after $v$ form a clique. A graph is chordal if,
and only if, it has a PEO~\cite{Fulkerson1965}. 
\begin{lemma}
  \label{lemma:peomao}
  A topological ordering $\tau$ of the vertices of a UCCG $G$ represents an AMO if, and only if, it is
  the reverse of a perfect elimination ordering.
\end{lemma}

Perfect elimination orderings can be computed in linear time with the
Lexicographic BFS algorithm~\cite{Rose1976} to which we will refer
as \emph{LBFS}. A modified version of this algorithm is presented as
Algorithm~\ref{alg:lbfs}. When called with $K =
\emptyset$ (and ignoring the
lines~\ref{line:startoutif}-\ref{line:endoutif}), it coincides with a
\emph{normal} LBFS. The modifications and the meaning of $K$ will
become clear later on.

\begin{algorithm}
  \caption{A modified version of the Lexicographic BFS~\cite{Rose1976}
    for computing the set $\chordalcomps_G(K)$. If the algorithm is executed
    with $K=\emptyset$, the algorithm performs a normal LBFS with
    corresponding traversal ordering $\tau$, which is the reverse of a
    PEO.}
  \label{alg:lbfs}
  \DontPrintSemicolon
  \SetKwInOut{Input}{input}\SetKwInOut{Output}{output}
  \SetKwFor{Rep}{repeat}{}{end}
  \Input{A UCCG $G = (V,E)$ and a clique $K\subseteq V$.}
  \Output{$\chordalcomps_G(K)$.}
  $\mathcal{S}\gets$ sequence of sets initialized with $(K,V\setminus K)$ \;
  $\tau\gets\text{empty list}$, $L\gets\emptyset$ \;
  \While{$\mathcal{S}$ is non-empty}{
    $X\gets\text{first non-empty set of $\mathcal{S}$}$ \;
    $\hbox to 0pt{$v$}\phantom{X}\gets\text{arbitrary vertex from $X$}$ \;  \label{line:findvertex}
    Add vertex $v$ to the end of $\tau$. \; \label{line:lbfsout}
    \If{$v$ is neither in a set in $L$ nor in $K$}{ \label{line:startoutif}
      $L\gets L\cup \{X\}$ {\color{red}}\;
      Output the undirected components of $G[X]$. \label{line:outsubs}
    } \label{line:endoutif}
    $X\gets X\setminus\{v\}$ \;
    Denote the current $\mathcal{S}$ by $(S_1, \dots, S_k)$. \;
    Replace each $S_i$  by $S_i \cap N(v), S_i
    \setminus N(v)$. \; \label{line:partrefine}
    Remove all empty sets from $\mathcal{S}$. \; 
  }
\end{algorithm}

LBFS runs in linear time (i.\,e., $\mathcal{O}(|V| +
|E|)$) when $\mathcal{S}$ is implemented as a doubly-linked list with a pointer to each vertex and
the beginning of each set.
The algorithm can be viewed as a fine-grained graph traversal compared to classical
breadth-first search (BFS), where the vertices are visited only by increasing distance to the start
vertex. LBFS keeps this property, but introduces additional
constraints on the ordering~$\tau$, in which the vertices are
visited ($\tau$ is called an \emph{LBFS ordering}). These
constraints guarantee that $\tau$ is the reverse of a
PEO. Hence, by
Lemma~\ref{lemma:peomao}, $\tau$ represents an AMO.

\begin{corollary}
   \label{cor:lbfsordermao}
   Every LBFS ordering $\tau$ of the vertices of a UCCG $G$ represents an AMO.
\end{corollary}

It holds even further that each AMO can be represented by at least one
LBFS ordering.

\begin{lemma}
  \label{lemma:allextlbfs}
  Every AMO of a UCCG $G$ can be represented by an LBFS ordering.
\end{lemma}

Each LBFS ordering starts with a maximal clique, because as long as
there is a vertex which can enlargen the current clique, the first set
of $\mathcal{S}$ is made up solely of such vertices.

\begin{lemma}
  \label{lemma:lbfsbegmc}
  Every LBFS ordering starts with a maximal clique.
\end{lemma}

These observations lead to the first idea in our algorithm for \hamo. We have seen that every
AMO can be represented by an LBFS ordering
(Lemma~\ref{lemma:allextlbfs}) and every LBFS ordering starts with
a maximal clique (Lemma~\ref{lemma:lbfsbegmc}). It follows:

\begin{corollary}
  \label{cor:startclique}
  Every AMO can be represented by a topological ordering which
  starts with a maximal clique.
\end{corollary}

This means that for us, it is \emph{sufficient
to consider topological orderings that start with a maximal clique}. 
Therefore, we generalize the definition of $s$-orientations:
We consider permutations $\pi$ of a clique $K$, as
each $\pi(K)$ represents a distinct AMO of the subgraph induced by $K$.
   
\begin{definition}
  Let $G = (V,E)$ be a UCCG, $K$ be a
  clique in $G$, and let $\pi(K)$ be a permutation of $K$.
  \begin{enumerate}
  \item The $\pi(K)$-orientation of $G$, also denoted $G^{\pi(K)}$,
    is the union of all AMOs of $G$ that can be represented by a
    topological ordering beginning with $\pi(K)$.
  \item Let $G^{K}$ denote the union of $\pi(K)$-orientations of $G$
    over all $\pi$, i.\,e.,\ let $G^{K} = \bigcup_{\pi} G^{\pi(K)}$.
  \item 
  Denote by $\cC_G(\pi(K))$ the undirected connected components 
  of $G^{\pi(K)}[V  \setminus K]$ and let 
  $\cC_G(K)$ denote the undirected connected components 
  of $G^{K}[V  \setminus K]$.
  \end{enumerate}
\end{definition}

Figure~\ref{fig:pikorients} shows an example $\pi(K)$-orientation of $G$:
For a graph $G$ in \textbf{(a)}, a clique $K = \{1,2,3,4\}$,
and a permutation $(4, 3, 2,1)$,  graph $G^{(4,3,2,1)}$ is presented in \textbf{(c)}.
It is the union of two DAGs which are AMOs of $G$, whose
topological orderings begin with $4,3,2,1$. The first DAG 
can be represented by topological ordering $4, 3, 2,1,5,6,7$ and 
the second one by $4, 3, 2,1,6,5,7$.  In Fig.~\ref{fig:pikorients},
we also compare the  $(4, 3, 2,1)$-orientation with an $s$-orientation, for 
$s=4$, shown in \textbf{(b)}. The undirected components of the orientations are
indicated by the colored regions. By orienting whole cliques at once, we get
significantly smaller undirected components in the resulting $\pi(K)$-orientation
than in the $s$-orientation (e.\,g., $\{5,6\}$ compared to $\{1, 2, 3, 5, 6\}$).
Finally, \textbf{(d)} illustrates graph $G^{\{1,2,3,4\}}$.
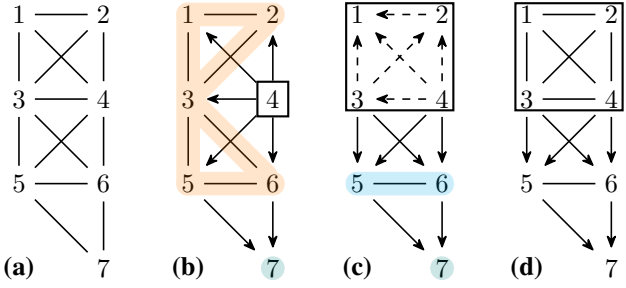
\begin{figure} 
  \centering
  \begin{tikzpicture}[scale=0.75]
    \node (1) at (0,-1.5) {$1$};
    \node (2) at (1.5,-1.5) {$2$};
    \node (3) at (0,-3) {$3$};
    \node (4) at (1.5,-3) {$4$};
    \node (5) at (0,-4.5) {$5$};
    \node (6) at (1.5,-4.5) {$6$};
    \node (7) at (1.5,-6) {$7$};
    \node (l1) at (0, -6) {\textbf{(a)}};
    \graph[use existing nodes, edges = {edge}] {
      1 -- {2,3,4};
      2 -- {3,4};
      3 -- {4,5,6};
      4 -- {5,6};
      5 -- {6,7};
      6 -- {7};
    };
    
    \node (1) at (3,-1.5) {$1$};
    \node (2) at (4.5,-1.5) {$2$};
    \node (3) at (3,-3) {$3$};
    \node[rectangle,draw, thick] (4) at (4.5,-3) {$4$};
    \node (5) at (3,-4.5) {$5$};
    \node (6) at (4.5,-4.5) {$6$};
    \node (7) at (4.5,-6) {$7$};
    \node (l2) at (3,-6) {\textbf{(b)}};

    \graph[use existing nodes, edges = {edge}] {
      1 -- {2,3};
      2 -- {3};
      3 -- {5,6};
      5 -- 6;
    };

    \graph[use existing nodes, edges = {arc}] {
      4 -> {1,2,3,5,6};
      5 -> 7;
      6 -> 7;
    };

    \begin{scope}[opacity = 0.2, transparency group]
      \filldraw[orange,rounded corners] (2.8,-4.7) rectangle
      (4.7,-4.3);
      \filldraw[orange,rounded corners] (2.8,-1.3) rectangle
      (3.2,-4.7);
      \filldraw[orange,rounded corners] (2.8,-1.3) rectangle (4.7,-1.7);
      \filldraw[orange,rounded corners] (2.85,-3.15) -- (3.15,-2.85) -- (4.65,-4.35) -- (4.35,-4.65);
      \filldraw[orange,rounded corners] (2.85,-2.85) -- (3.15,-3.15) -- (4.65,-1.65) -- (4.35,-1.35);
    \end{scope}
    \filldraw[teal, opacity = 0.2] (4.5,-6) circle (0.2);

    \node (1) at (6,-1.5) {$1$};
    \node (2) at (7.5,-1.5) {$2$};
    \node (3) at (6,-3) {$3$};
    \node (4) at (7.5,-3) {$4$};
    \node (5) at (6,-4.5) {$5$};
    \node (6) at (7.5,-4.5) {$6$};
    \node (7) at (7.5,-6) {$7$};
    \node (l3) at (6, -6) {\textbf{(c)}};

    \graph[use existing nodes, edges = {edge}] {
      5 -- 6;
    };

    \graph[use existing nodes, edges = {arc}] {
      2 ->[dashed] 1;
      3 ->[dashed] {1,2};
      3 -> {5,6};
      4 ->[dashed] {1,2,3};
      4 -> {,5,6};
      5 -> 7;
      6 -> 7;
    };

    \draw[thick] (5.8,-1.3) rectangle (7.7,-3.2);

    \filldraw[cyan, opacity = 0.2, rounded corners] (5.8,-4.7) rectangle (7.7,-4.3);
    \filldraw[teal, opacity = 0.2] (7.5,-6) circle (0.2);
    
    \node (1) at (6+3,-1.5) {$1$};
    \node (2) at (7.5+3,-1.5) {$2$};
    \node (3) at (6+3,-3) {$3$};
    \node (4) at (7.5+3,-3) {$4$};
    \node (5) at (6+3,-4.5) {$5$};
    \node (6) at (7.5+3,-4.5) {$6$};
    \node (7) at (7.5+3,-6) {$7$};
    \node (l3) at (6+3, -6) {\textbf{(d)}};

    \graph[use existing nodes, edges = {edge}] {
      5 -- 6;
      2 -- 1;
      3 -- {1,2};
      4 -- {1,2,3};
    };

    \graph[use existing nodes, edges = {arc}] {
      3 -> {5,6};
      4 -> {,5,6};
      5 -> 7;
      6 -> 7;
    };

    \draw[thick] (5.8+3,-1.3) rectangle (7.7+3,-3.2);

    
  \end{tikzpicture}
  \caption{For a UCCG $G$ in \textbf{(a)}, the figure shows $G^{(4)}$ in 
    \textbf{(b)},  $G^{(4,3,2,1)}$ in \textbf{(c)}, and $G^{\{1,2,3,4\}}$ in \textbf{(d)}.
    The undirected components in $G^{(4)}$ and $G^{(4,3,2,1)}$ are indicated by
    the colored regions and the vertices put at the beginning of the
    topological ordering by a rectangle (all edges from the rectangle
    point outwards). Edges inside the rectangle in \textbf{(c)} are dashed, as they
    have no influence on the further edge directions outside the rectangle. 
    }
  \label{fig:pikorients}
\end{figure}

The undirected components of the $\pi(K)$-orientation are chordal
graphs, which can be oriented independently, yielding the following
recursive formula: 

\begin{lemma}\label{lemma:independentOrientation}
  The undirected connected components in 
  $\cC_G(\pi(K))$ are chordal and it holds that:
  \[
    \hamo(G^{\pi(K)}) = \prod_{H \in \cC_G(\pi(K))} \hamo(H).
  \]
\end{lemma}

The crucial observation is that the undirected components
$\cC_G(\pi(K))$ are independent of the permutation
$\pi$. This means no matter how the vertices $\{1, 2, 3, 4\}$ are
permuted, if the whole clique is put at the beginning of the
topological ordering, no further edge orientations will be
influenced. Informally, this is because all edges from the clique $K$
to other vertices are directed outwards no matter the permutation
$\pi$. 
We formalize this observation in the following: 

\begin{proposition}
  \label{prop:cliqueid}
  Let $G$ be a UCCG and $K$ be a clique of $G$.
  For each permutation $\pi(K)$ it is true that all edges of $G^{\pi(K)}$
  coincide with the edges of $G^K$, excluding the edges connecting 
  the vertices in $K$. 
  Hence, $\cC_G(\pi(K)) = \cC_G(K)$ and it holds that:
  \[
    \sum_{\text{$\pi$ over $K$}} \hamo(G^{\pi(K)}) = |K|! \times \prod_{\balap{H \in \chordalcomps_G(K)}} \hamo(H).
  \]
\end{proposition}

This is the key property that allows us to efficiently deal with
whole cliques at once, instead of considering single vertices
one-by-one. As each permutation
$\pi(K)$ represents a distinct AMO of $G[K]$, this formula indeed computes the number of AMOs, which can be
represented by a topological ordering with clique $K$ at the
beginning. However, we are not done yet, as there are some further obstacles we need to
overcome in order to obtain a polynomial-time algorithm for \hamo,
which are dealt with in the following sections.

Before that, we leverage the connection between AMOs and
LBFS orderings one more time, to propose a linear-time algorithm for computing
$\chordalcomps_G(K)$~--~the full Algorithm~\ref{alg:lbfs}. This algorithm performs
an LBFS and, whenever a vertex could be picked for the first time, the
corresponding first set in $\mathcal{S}$ is appended 
to $L$ and the
undirected components of the set are output
(lines~\ref{line:startoutif}-\ref{line:endoutif}).
For instance, in the example above, after the vertices in $K$ are visited, we have $\mathcal{S} =
(\{5,6\}, \{7\})$. As $\{5,6\}$ is currently the first set of
$\mathcal{S}$ and vertices $5$ and $6$ are not in any
set in~$L$ yet ($L$ is still empty), the set $\{5,6\}$ is appended to
$L$ and $5-6$ is output as an element of $\chordalcomps_G(K)$. 

\begin{theorem} \label{thm:lintimesubp}
  For a chordal graph $G$ and a clique $K$,  Algorithm~\ref{alg:lbfs}
  computes $\chordalcomps_G(K)$ in time
  $\mathcal{O}(|V| + |E|)$. 
\end{theorem}

\begin{proof}[Sketch of Proof]
  When a set is appended to  $L$, each vertex of this set could have been the next chosen vertex
  in line~\ref{line:findvertex}. Thus, all edges between the
  vertices in a set in $L$ may occur as either $u \rightarrow v$ (if $u$ is chosen
  next) or as $u \leftarrow v$ (if $v$ is chosen next) in an
  AMO having $K$ at the beginning of the LBFS ordering. 
  As a  $\pi(K)$-orientation of $G$ is the union of all corresponding
  AMOs, we have $u - v$.
  
  If, on the other hand, $u$ and $v$ are neighbors but not in the same
  set in $L$, the edge between them is oriented $u
  \rightarrow v$ in $G^{K}$, assuming $u$ is visited before $v$. This is due to
  an inductive argument by which $u \rightarrow v$ follows
  from iterative application of the first Meek rule.
\end{proof}

Both cases dealt with in the proof sketch can be seen in
Fig.~\ref{fig:pikorients}. The edge $5 - 6$ remains undirected as
either vertex could be chosen first in Algorithm~\ref{alg:lbfs}, while
we have $5 \rightarrow 7$, because the first Meek rule applies to
$\rightarrow 5 - 7$. 

We note that this algorithm could also be used for finding the UCCGs
of the $s$-orientations of a chordal graph in linear time, which
improves upon prior work~\cite{He2015,Ghassami2019,Talvitie2019}.

\section{Counting AM-Orientations with \\ Minimal Separators and Maximal Cliques}\label{sec:separators}
Using the insights from the previous section, we would like to count
the AMOs of a chordal graph $G$ with the following
recursive procedure based on Proposition~\ref{prop:cliqueid}: Pick a
maximal clique $K$, consider all its
permutations at once, and take the product of the recursively computed number of
AMOs of the UCCGs of $\cC_G(K)$. By
Corollary~\ref{cor:startclique}, we will count every AMO in
this way, if we compute the sum over all maximal
cliques. Unfortunately, we will count some orientations multiple
times, as a single AMO can be represented by multiple topological
orderings.  For instance, assume we have two maximal cliques $K_1$ and
$K_2$ with $K_1\cap K_2=S$ such that $K_1\setminus S$ is separated
from $K_2\setminus S$ in $G[V\setminus S]$. A topological ordering
that starts with $S$ can proceed with either $K_1\setminus S$ or
$K_2\setminus S$ and result in the same AMO.

\begin{example}\label{example:coreSplit}
  Consider the following chordal graph (left) with maximal cliques
  $K_1=\{1,2,3\}$ and $K_2=\{2,3,4\}$. A possible
  AMO of the graph is shown on the right.
  
  
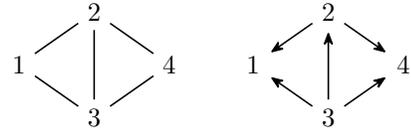
\begin{center}
    \tikz[yscale=0.7]{
      \rhombus
      \graph[use existing nodes, edges = {edge}] {
        1 -- 2 -- 4 -- 3 -- 1; 2 -- 3;
      };
    }\qquad
    \tikz[yscale=0.7]{
      \rhombus
      \graph[use existing nodes, edges = {arc}] {
        2 -> {1, 4};
        3 -> {1, 2, 4};
      };
    }
  \end{center}

  The AMO has two topological orderings: $\tau_1=(3,2,1,4)$ and
  $\tau_2=(3,2,4,1)$ starting with $K_1$ and $K_2$, respectively.
  Hence, if we count all topological orderings starting with
  $K_1$ and all topological orderings starting with $K_2$, we will count the
  AMO twice. However, $\tau_1$ and $\tau_2$ have $3,2$ as common
  prefix and $K_1\cap K_2=\{2,3\}$ is
  a minimal separator of the graph~--~a fact that we will use
  in the following.\exampleqed
\end{example}

\begin{lemma}
  \label{lemma:minimalSeparator}
  Let $\alpha$ be an AMO of a chordal graph~$G$ and $\tau_1$,
  $\tau_2$ be two topological orderings that represent $\alpha$. Then
  $\tau_1$ and $\tau_2$ have a common prefix $S\in\separators(G)\cup\cliques(G)$.
\end{lemma}

Note that this lemma implies that \emph{all} topological orderings
that correspond to an AMO have a common prefix, which is a minimal
separator or maximal clique.

The combinatorial function $\phi$, as defined below, 
plays an important role to avoid overcounting.

\begin{definition}\label{def:phi}
For a set $S$ and a collection $R$ of subsets of~$S$, 
we define $\phi(S,R)$ as the number of
all permutations of $S$ that do not have 
a set $S'\in R$ as prefix. 
\end{definition}
\begin{example}\label{example:phi}
Consider the set $S= \{2,3,4,5\}$ and the collection $R=\big\{\{2,3\},$ $\{2,3,5\}\big\}$. Then 
$\phi(S,R) = 16$ since  there are 16 permutations of
$\{2,3,4,5\}$ that neither start with $\{2,3\}$ nor
$\{2,3,5\}$~--~e.\,g., $(3,2,4,5)$ and $(2,5,3,4)$ are forbidden as
they start with $\{2,3\}$ and $\{2,3,5\}$, respectively; but $(3,5,4,2)$ is allowed.\exampleqed
\end{example}

In this paper, we always consider sets $S \in \separators(G) \cup \cliques(G)$
and collections $R \subseteq \separators(G)$. Therefore, we can use the abbreviation
$\phi(S)=\phi\big(S,\{\,S'\mid S'\in\Delta(G)\wedge S'\subsetneq
S\,\}\big)$.

\begin{proposition}\label{proposition:countingFormula}
Let $G$ be a UCCG. Then:
  \[\hamo(G)=\sum_{S\in\Delta(G)\cup\Pi(G)}
  \phi(S)
  \times 
  \prod_{\balap{H\in\chordalcomps_G(S)}}\hamo(H).
  \]
\end{proposition}
\begin{proof}[Sketch of Proof]
  By Lemma~\ref{lemma:minimalSeparator}, every AMO can be
  represented by a topological ordering that starts with a vertex set $S\in\separators(G)\cup\cliques(G)$. The definition of $\phi(S)$ and an
  induction over the product yields the claim.
\end{proof}

\begin{example}\label{example:separators}
We consider the following chordal graph with
two minimal separators 
and three maximal cliques:
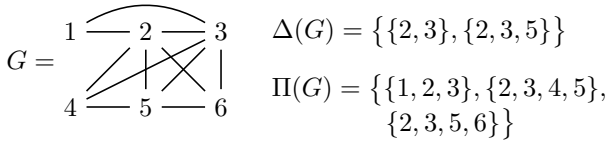
\begin{center}
  $G=$\tikz[baseline={(0,-0.5)}]{\marcelgraph}%
  \quad%
  \tikz[baseline={(0,-0.5)}]{
    \node[anchor=west, baseline] at (0,0)  {$\separators(G)=\big\{\{2,3\}, \{2,3,5\}\big\}$};
    \node[anchor=west, baseline, text width=4.5cm] at (0,-1) {$\cliques(G)=\big\{\{1,2,3\}, \{2,3,4,5\},$\\\hspace{1.4cm} $\{2,3,5,6\}\big\}$};
  }
\end{center}
To compute $\hamo(G)$ using Proposition~\ref{proposition:countingFormula}, we need the following
values. Note that the resulting subgraphs $H$ are trivial, except for the
case $S=\{2,3\}$ and $S=\{1,2,3\}$. In these cases, we obtain the
induced path on $\{4,5,6\}$, which has three possible AMOs.
\begin{center}
  \begin{tabular}{ccc}
    \emph{$S\in\separators(G)\cup\cliques(G)$} & \emph{$\phi(S)$} & $\prod\limits_{H\in\cC_G(S)} \hamo(H)$\\[1.5ex]
    $\{2,3\}$      & 2  & 3  \\
    $\{2,3,5\}$    & 4  & 1  \\
    $\{1,2,3\}$    & 4  & 3  \\
    $\{2,3,4,5\}$  & 16 & 1  \\
    $\{2,3,5,6\}$  & 16 & 1  \\
  \end{tabular}
\end{center}
Using Proposition~\ref{proposition:countingFormula} we can compute $\hamo(G)$ as follows:
\[
  \hamo(G) = 2\cdot 3 + 4\cdot 1 + 4\cdot 3 + 16\cdot 1 + 16\cdot 1 = 54.
\]
We remark that we do \emph{not} have discussed how to
compute~$\phi(S)$ yet~--~for this example, this can be done by naive
enumeration. In general, however, this is a non-trivial task. We tackle this
issue in the next section.  \exampleqed
\end{example}

\section{The Clique-Picking Algorithm} \label{sec:cliquepicking}
In the previous section, we showed how to count
AMOs by using minimal separators in
order to avoid overcounting. It is rather easy to check that we can
compute $\phi(S,R)$ in time \emph{exponential} in $|R|$ using the
inclusion-exclusion principle. However, our goal is \emph{polynomial
  time} and, thus, we have to restrict the collection $R$.
\begin{lemma}\label{lemma:efficientPhi}
  Let $S$ be a set and $R=\{X_1,\dots,X_{\ell}\}$ be a collection of
  subsets of $S$ with $X_1\subsetneq X_2\subsetneq\dots\subsetneq
  X_{\ell}$. Then:
  \[
    \phi(S,R) = |S|!
    -\sum_{i=1}^{\ell}|S\setminus X_i|!\cdot\phi(X_i,\{X_1,\dots,X_{i-1}\}).    
  \]
\end{lemma}
Observe that this formula can be evaluated in polynomial time with
respect to $|S|$ and $\ell$, as all occurring subproblems have the
form $\phi(X_i,\{X_1,\dots,X_{i-1}\})$ and, thus, there are at most
$\ell$ of them. The goal of this section is to develop a version of
Proposition~\ref{proposition:countingFormula} based on this lemma.

To achieve this goal, we rely on the strong structural properties that
chordal graphs entail: A
\emph{rooted clique tree} of a UCCG~$G$ is a triple $(T,r,\iota)$ such
that $(T,r)$ is a rooted tree and
$\iota\colon V_T\rightarrow \cliques(G)$ a bijection between the
nodes of $T$ and the maximal cliques of $G$ such that
$\{\,x\mid v\in\iota(x)\,\}$ is connected in $T$ for all $v\in V_G$.
In slight abuse of notation, we denote, for a set $C\subseteq V_G$, by
$\iota^{-1}(C)$ the subtree $\{\,x\mid C\subseteq\iota(x)\,\}$. We
denote the children of a node $v$ in a tree~$T$ by
$\text{children}_T(v)$.  It is well-known that (i)~every chordal graph
has a rooted clique tree $(T,r,\iota)$ that can be computed in linear
time, and (ii)~a set $S\in V_G$ is a minimal separator if, and only
if, there are two adjacent nodes $x,y\in V_T$ with
$\iota(x)\cap\iota(y)=S$~\cite{Blair1993}.

We wish to interleave the structure provided by the clique tree with a
formula for computing $\hamo$. For this sake, let us define the
\emph{forbidden prefixes} for a node $v$ in a clique tree.

\begin{definition}
  Let $G$ be a UCCG, $\mathcal{T} = (T, r, \iota)$ a rooted clique tree of
  $G$, $v$ a node in $T$ and $r = x_1, x_2, \dots, x_p = v$ the unique
  $r$-$v$-path. We define the set $\mathrm{FP}(v, \mathcal{T})$ to contain
  all sets $\iota(x_i) \cap \iota(x_{i+1}) \subseteq \iota
  (v)$ for $1 \leq i < p$. 
\end{definition}

\begin{lemma}\label{lemma:inputPhi}
  We can order the elements of the set $\mathrm{FP}(v, \mathcal{T})$ as $X_1
  \subsetneq X_2 \subsetneq \dots \subsetneq X_\ell$.
\end{lemma}

By combining the 
lemma with Lemma~\ref{lemma:efficientPhi}, we
deduce that  $\phi(\iota(v),\mathrm{FP}(v,\mathcal{T}))$ can be evaluated in
polynomial time for nodes $v$ of the clique tree. We are left with the task of developing a formula for
$\hamo$ in which all occurrences of $\phi$ are of this form. It is
quite easy to come up with such formulas that count every AMO at least
once~--~but, of course, we have to ensure that we count every AMO
\emph{exactly} once.

To ensure this property, we introduce for every AMO~$\alpha$ a partial
order~$\prec_{\alpha}$ on the maximal cliques. Then we prove that
there is a unique minimal element with respect to this order, and
deduce a formula for $\hamo$ that counts~$\alpha$ only ``at this
minimal element''.
To get started, we need a technical definition and some auxiliary
lemmas that give us more control over the rooted clique tree. 

\begin{definition} 
  An \emph{$S$-flower} for a minimal separator~$S$ is a maximal set
  $F\subseteq\{\,K\mid K\in\cliques(G)\wedge S\subseteq K\,\}$ such that $\bigcup_{K\in F}K$ is
  connected in $G[V\setminus S]$. The \emph{bouquet}~$\bouquet(S)$ of a
  minimal separator $S$ is the set of all $S$-flowers.
\end{definition}
\begin{example}
  The $\{2,3\}$-flowers of the graph from Example~\ref{example:separators} are
  $\{\{1,2,3\}\}$ and $\{\{2,3,4,5\}, \{2,3,5,6\}\}$.
  \exampleqed
\end{example}
\begin{lemma}
  \label{lemma:connflowers}
  An $S$-flower $F$ is a connected subtree in a rooted clique tree $(T,r,\iota)$.
\end{lemma}

\begin{lemma}
  \label{lemma:bouquetpart}
  For any minimal separator $S$, the bouquet $\bouquet(S)$ is a
  partition of $\iota^{-1}(S)$.
\end{lemma}

Since for a $S\in\Delta(G)$ the subtree $\iota^{-1}(S)$ of
$(T,r,\iota)$ is connected,
Lemma~\ref{lemma:connflowers} and Lemma~\ref{lemma:bouquetpart} give
rise to the following order on $S$-flowers $F_1,F_2\in\bouquet(S)$:
$F_1\prec_T F_2$ if $F_1$ contains a node on the unique path
from $F_2$ to the root of $T$.

\begin{lemma}\label{lemma:flowersAreOrdered}
  There is a unique least $S$-flower in $\bouquet(S)$ with respect to $\prec_T$.
\end{lemma}

The lemma states that for every
AMO~$\alpha$ there is a flower $F$ at which we want to count
$\alpha$. We have to be sure that this is possible, i.\,e., that a
clique in $F$ can be used to generate~$\alpha$.

\begin{lemma}
  \label{lemma:BSeveryAMO}
  Let $\alpha$ be an AMO such that every topological
  ordering that represents $\alpha$ has the minimal separator~$S$ as
  prefix. Then every $F\in\bouquet(S)$ contains a clique $K$ such that
  there is a $\tau\in\mathrm{top}(\alpha)$ starting with $K$.
\end{lemma}

We use $\prec_T$ to define, for a fixed AMO~$\alpha$, a partial
order~$\prec_{\alpha}$ on the set of maximal cliques, which are at the
beginning of some $\tau\in\mathrm{top}(\alpha)$, as follows:
$K_1\prec_{\alpha}K_2$ if, and only if,
(i)~$K_1\cap K_2=S\in\separators(G)$, (ii)~$K_1$ and $K_2$ are in
$S$-flowers $F_1,F_2\in\bouquet(S)$, respectively, and
(iii)~$F_1\prec_T F_2$.

\begin{proposition}\label{proposition:fpFormula}
  Let $G$ be a UCCG and $\mathcal{T} = (T, r, \iota)$ be a rooted
  clique tree of $G$. Then:
  \[
    \hamo(G) = \sum_{\balap{v\in V_T}} \phi(\iota(v), \mathrm{FP}(v, \mathcal{T})) 
    \times \prod_{\balap{H \in \chordalcomps_G(\iota(v))}} \hamo(H).
  \]
\end{proposition}
\begin{proof}[Sketch of Proof]
  First, prove that for every AMO~$\alpha$ there is a unique least
  $K\in\cliques(G)$ with respect to~$\prec_{\alpha}$. Let $x$ be the
  node of the clique tree with $\iota(x)=K$, we deduce that $\alpha$
  is counted in $\phi(\iota(x), \mathrm{FP}(x,\mathcal{T}))$, but
  is blocked in all other nodes $y$ by some set in $\mathrm{FP}(y,\mathcal{T})$.
\end{proof}

Algorithm~\ref{alg:cliquepicking} evaluates this formula, utilizing
memoization to avoid recomputations. Traversing the clique tree with a BFS allows for simple
computation of $\mathrm{FP}$.

\begin{algorithm}
  \caption{The Clique-Picking algorithm computes the number of
    acyclic moral orientations of a UCCG $G$.}
  \label{alg:cliquepicking}
  \SetKwInOut{Input}{input}\SetKwInOut{Output}{output}
  \DontPrintSemicolon
  \Input{A UCCG $G = (V,E)$.}
  \Output{$\hamo(G)$.}
  \SetKwFunction{FCount}{count}
  \SetKwFunction{FNewors}{neworients}
  \SetKwFunction{FPop}{pop}
  \SetKwFunction{FAppend}{append}
  \SetKwFunction{FPush}{push}
  \SetKwProg{Fn}{function}{}{end}
  \Fn{\FCount{$G$, $\mathrm{memo}$}}{
    \If{$G \in \mathrm{memo}$}{\KwRet $\mathrm{memo}[G]$}
    $\mathcal{T} = (T,r,\iota) \leftarrow$ a rooted clique tree of $G$ \;
    $\mathrm{sum} \leftarrow 0$ \;
     $Q \leftarrow$ queue with single element $r$ \;
    \While{$Q$ is not empty}{
      $v \leftarrow \FPop(Q)$ \;\label{line:pop}
      $\FPush(Q, \mathrm{children}(v))$ \;
      $\mathrm{prod} \leftarrow 1$ \;
      \ForEach{$H \in \chordalcomps_G(\iota(v))$}{
        $\mathrm{prod} \leftarrow \mathrm{prod} \cdot \FCount(H, \mathrm{memo})$ \; \label{line:mult}
      }
      $\mathrm{sum} \leftarrow \mathrm{sum} +
      \phi(\iota(v),\mathrm{FP}(v,\mathcal{T})) \cdot \mathrm{prod}$
      \; \label{line:weightedsum}
    }
    $\mathrm{memo}[G] = \mathrm{sum}$ \;
    \KwRet $\mathrm{sum}$\;
  }
\end{algorithm}

\begin{theorem}\label{theorem:cliquepicking}
  For an input UCCG $G$, Algorithm~\ref{alg:cliquepicking} returns
  the number of AMOs of $G$.
\end{theorem}

\begin{example}\label{example:cliquePicking}
  We consider a rooted clique tree~$(T,r,\iota)$ for the
  graph $G$ from Example~\ref{example:separators}. The root is labeled
  with $r$ and the function $\iota$ is visualized in
  \textcolor{ba.blue}{blue}. The edges of the clique tree are
  labeled with the corresponding minimal separators.

  \begin{center}
    \tikzset{
      tree node/.style = {
        draw,
        fill,
        circle,
        inner sep     = 0pt,
        minimum width = 1mm,
        outer sep     = 5pt,
      }
    }
    
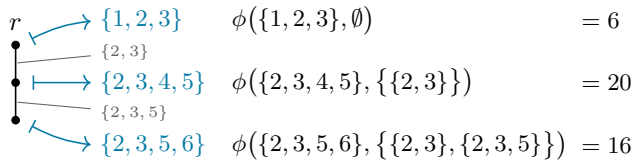
\tikz{
      \node[inner sep=0pt, minimum width=1mm, circle, label={[label distance=0.2mm]$r$}] at (0,0) {};
      \node[tree node] (r) at (0,  0) {};
      \node[tree node] (a) at (0,-.5) {};
      \node[tree node] (b) at (0, -1) {};
      \draw[edge] (0,0) -- (0, -.5) -- (0,-1);

      \node[anchor=west, color=ba.blue, font=\small] (c1) at (1,  0.33)   {$\{1,2,3\}$};
      \node[anchor=west, color=ba.blue, font=\small] (c2) at (1, -0.5)    {$\{2,3,4,5\}$};
      \node[anchor=west, color=ba.blue, font=\small] (c3) at (1, -1.33)   {$\{2,3,5,6\}$};
      \graph[use existing nodes, edges = {edge, |->, color=ba.blue}] {
        r ->[bend left=10]  c1;
        a -> c2;
        b ->[bend right=10] c3;
      };

      \node[anchor=west, color=ba.gray!65!black, font=\tiny] (s1) at (1, -0.1)   {$\{2,3\}$};
      \node[anchor=west, color=ba.gray!65!black, font=\tiny] (s2) at (1, -0.9)  {$\{2,3,5\}$};
      \node (e1) at (-0.1, -0.25) {};
      \node (e2) at (-0.1, -0.75) {};
      \graph[use existing nodes, edges = {thin, color=ba.gray!65!black}] {
        s1 -- e1;
        s2 -- e2;
      };

      \node[anchor=west, font=\small] at (2.75,  0.33) {$\phi\big(\{1,2,3\},   \emptyset\big)$};
      \node[anchor=west, font=\small] at (2.75, -0.5)  {$\phi\big(\{2,3,4,5\}, \big\{\{2,3\}\big\}\big)$};
      \node[anchor=west, font=\small] at (2.75, -1.33) {$\phi\big(\{2,3,5,6\}, \big\{\{2,3\},\{2,3,5\}\big\}\big)$};
      \node[anchor=west, font=\small] at (7.4,  0.33) {$=6$};
      \node[anchor=west, font=\small] at (7.4, -0.5)  {$=20$};
      \node[anchor=west, font=\small] at (7.4, -1.33) {$=16$};
    }
  \end{center}
  Algorithm~\ref{alg:cliquepicking} traverses the tree $T$ from the root
  $r$ to the bottom and computes the values shown at the right. 
  The only case in which we
  obtain a non-trivial subgraph is for $S=\{1,2,3\}$ (an
  induced path on $\{4,5,6\}$). Therefore:
  \begin{equation}
  \hamo(G)=6\cdot 3 + 20\cdot 1 + 16\cdot 1 = 54.\tag*{\exampleqed}
  \end{equation}
\end{example}

Since clique trees can be computed in linear time~\cite{Blair1993}, an
iteration of the algorithm runs in polynomial time due to
Lemma~\ref{lemma:efficientPhi} and~\ref{lemma:inputPhi}.  We prove in
the next section that Algorithm~\ref{alg:cliquepicking} performs at most
$2\cdot|\cliques(G)|-1$ recursive calls, which implies overall
polynomial run time.

%
%
\begin{figure*}[!htbp]
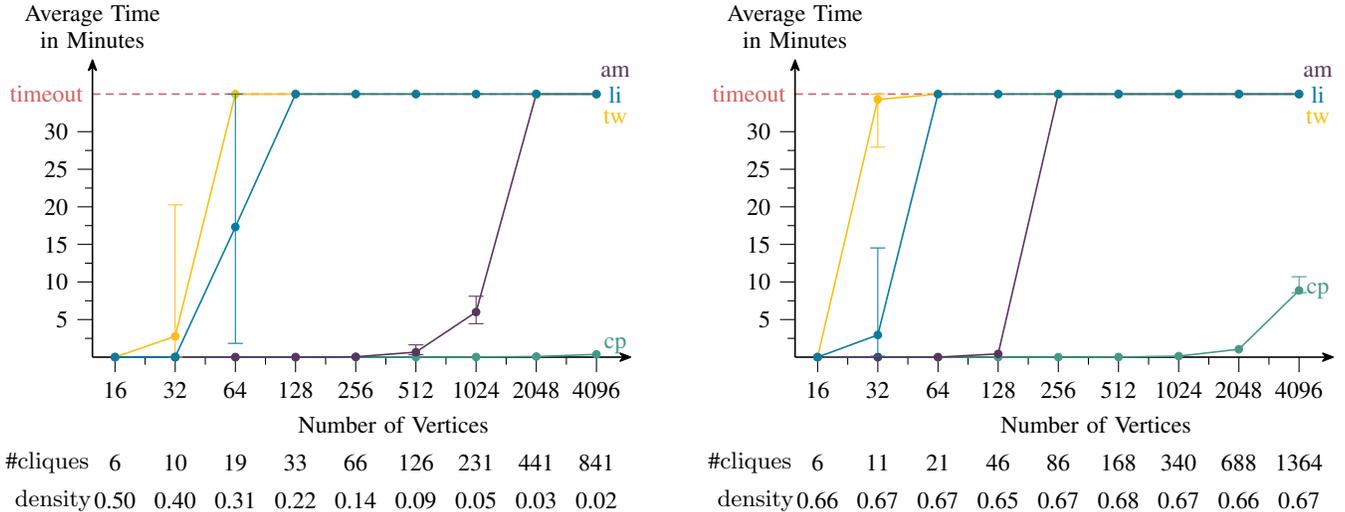

  \input{img/subtree-logn-wo-headline.tikz}\quad\input{img/interval-wo-headline.tikz}
  \caption{Experimental results for the solvers
    \textcolor{ba.pine}{Clique-Picking (cp)},
    \textcolor{ba.violet}{AnonMAO (am)}, 
    \textcolor{ba.yellow}{TreeMAO (tw)}, and
    \textcolor{ba.blue}{LazyIter (li)} on random chordal graphs with
    $n = 16, 32, \dots, 4096$ vertices. For the left plot, we used
    graphs generated with the subtree intersection method and density
    parameter $k = \log{n}$; the right plot contains the results for
    random interval graphs. At the bottom, we present the number of
    maximal cliques as well as the graph density $|E|/\binom{|V|}{2}$.
  }
  \label{fig:exp}
\end{figure*}

\section{The Complexity of \#\kern-0.5ptAMO}
\label{sec:compl}
We analyze the run time of the Clique-Picking algorithm by bounding
the number of connected chordal subgraphs that we encounter. The
following proposition shows that this number can be bounded by
$\mathcal{O}(|\cliques(G)|)$. Recall that we have $|\cliques(G)| \leq |V|$ in
chordal graphs and, thus, we only have to handle a linear number of
recursive calls.

\begin{proposition} \label{prop:subpbound}
  Let $G$ be a UCCG. The number of distinct UCCGs explored by \texttt{count} is bounded by $2|\cliques(G)|-1$. 
\end{proposition}
\begin{proof}[Sketch of Proof]
  We observe that there is a bijection between $S$-flowers of $G$ and
  distinct UCCGs. The linear bound
  follows, as a separator $S$, that has a bouquet of size $k$, is associated
  with $k-1$ edges of the clique tree. Therefore, we have at most
  $|\cliques(G)|-1 - (k-1)$ further separators and the maximum number of
  flowers is obtained if the quotient $k / (k-1)$ is
  maximized~--~which is the case for $k=2$.
\end{proof}

We are now able to bound the run time of Clique-Picking:
\begin{theorem}
  \label{thm:runtime}
  The Clique-Picking algorithm runs in time
  $\mathcal{O}\big(\,|\cliques(G)|^2 \cdot (|V| + |E|)\,\big)$.
\end{theorem}
\begin{proof}[Sketch of Proof]
  The algorithm explores $\mathcal{O}(|\cliques(G)|)$ distinct UCCGs
  by Proposition~\ref{prop:subpbound}. For each UCCG we compute a
  clique tree, and for all nodes $v\in V_T$ the set $\cC_G(\iota(v))$
  and the value $\phi(\iota(v),\mathrm{FP}(v, \mathcal{T}))$. Both can
  be done in time
  $\mathcal{O}(|V| + |E|)$ by Theorem~\ref{thm:lintimesubp} and by using the formula from
  Lemma~\ref{lemma:efficientPhi}. 
\end{proof}  

As one would expect, the Clique-Picking algorithm can~--~with slight
modifications~--~also be used to sample Markov
equivalent DAGs uniformly at random. Hence, this problem can be solved
in polynomial time, too.


\begin{theorem}\label{theorem:samplingtime}
  There is an algorithm that, given a connected chordal graph $G$, uniformly
  samples AMOs of $G$ in time $\mathcal{O}(|V|+|E|)$ after an initial
  $\mathcal{O}(\cliques(G)^2\cdot |V| \cdot |E|)$ setup.  
\end{theorem}
\begin{proof}[Sketch of Proof]
  We sample the AMOs recursively: Draw a clique $K$
  proportional to the number of AMOs counted at this clique; 
  uniformly draw a permutation of $K$ that does not start with a
  forbidden prefix; recurse on subgraphs.
\end{proof}

We summarize the findings of this section:\footnote{Additionally, we remark that the size of interventional Markov equivalence classes
  can also be computed in polynomial time. 
}

\begin{theorem}\label{theorem:main}
  The problems $\hamo$ and uniform sampling from a Markov
  equivalence class are in $\P$.
\end{theorem}

The following theorem shows that Theorem~\ref{theorem:main} is tight
in the sense that counting Markov equivalent
DAGs that encode additional background knowledge (e.\,g., that are represented
as so-called \emph{PDAGs} or \emph{MPDAGs}) is not in $\P$ under standard complexity-theoretic
assumptions.

\begin{theorem}
  The problem of counting the number of AMOs is
  \sharpP-complete for PDAGs and MPDAGs. 
\end{theorem}
\section{Experimental Evaluation of Clique-Picking}
\label{sec:exp}
We evaluate the practical performance of the
Clique-Picking algorithm by comparing it to three state-of-the-art
algorithms for $\hamo$. \emph{AnonMAO}~\cite{Ganian2020} is the best root-picking
method; \emph{TreeMAO}~\cite{Talvitie2019} utilizes dynamic
programming on the clique tree; and \emph{LazyIter}~\cite{Teshnizi20}
combines techniques from intervention design with dynamic programming.

Figure~\ref{fig:exp} shows the run time of the four algorithms on
random chordal graphs~--~details of the random graph generation and
further experiments can be found in the supplementary
material\footnote{The source code and supplementary material are
  available at https://github.com/mwien/CliquePicking.}. We chose the
random subtree intersection method (left
plot in Fig.~\ref{fig:exp}) as it generates a broad range of chordal
graphs~\cite{SekerHET17}; and we complemented these with random interval graphs (right
plot) as \emph{AnonMAO} runs provably in polynomial
time on this subclass of chordal graphs~\cite{Ganian2020}.

The Clique-Picking algorithm outperforms its competitors in both
settings. For the subtree intersection graphs, it solves all instances
in less than a minute, while the other solvers are not able to solve
instances with more than 1024 vertices. The large instances of the interval graphs
are more challenging, as they are denser and have more maximal
cliques. However, Clique-Picking is
still able to solve all instances, while the best competitor,
\emph{AnonMAO}, can not handle graphs with 256 or more vertices. 

\section{Conclusion}
We presented the first polynomial-time algorithms for counting and
sampling Markov equivalent DAGs. Our novel Clique-Picking approach
utilizes the clique tree without applying cumbersome dynamic
programming on it. As a result, the algorithm is not only of theoretical but also of high
practical value, being the fastest algorithm by a large margin. 

\section*{Acknowledgements}
This work was supported by the Deutsche Forschungsgemeinschaft (DFG)
grant LI634/4-2.

The authors thank Paula Arnold for her help in setting up the experiments.
\bibliography{countingmec}
\clearpage


\begin{strip}
  \centering
  \textbf{\huge Appendix}
\end{strip}

\part{Detailed Proofs}
We present the detailed proofs that are missing within the main
paper. This part of the appendix is structured as the main paper,
i.\,e., for every section of the paper (that contains lemmas or
theorems) there is a section here that contains the corresponding
detailed proofs. For the reader's convenience, we repeated the
statements of the lemmas and theorems and, if appropriate, recall some
central definitions. Some of the proofs require additional auxiliary lemmas
that did not appear within the main text. These new lemmas are
marked with a~$\star$.

\section*{Proofs of Section 3:\\ Lexicographic BFS and AMOs}
\begin{claim lemma}
   A topological ordering $\tau$ of the vertices of a UCCG $G$ represents an AMO if, and only if, it is
   the reverse of a perfect elimination ordering.
\end{claim lemma}

\begin{proof}
   For the first direction, assume $\tau$ is a topological ordering representing an
   AMO. By definition of AMOs, there can not be a v-structure and, thus,
   if two vertices $x,y \in N(u)$ precede $u$
   in $\tau$, they need to be neighbors. This implies that the neighbors
   of $u$ preceding $u$ in $\tau$ form a clique. Thus, the reverse
   of $\tau$ is a perfect elimination ordering.

   For the second direction, assume $\rho$ is a perfect elimination ordering and orient the edges
   according to the topological ordering that is the reverse of
   $\rho$. Clearly, the orientation is acyclic. Moreover, there can be
   no v-structure, as two vertices $x,y$ preceding $u$ in the reverse
   of $\rho$ are neighbors. Thus, the reverse of $\rho$ represents an AMO.
\end{proof}

\begin{claim corollary}
   Every LBFS ordering $\tau$ of a UCCG $G$ represents an AMO.
\end{claim corollary}

\begin{proof}
  Follows immediately from Lemma~\ref{lemma:peomao} and the fact
  that an LBFS always outputs a PEO when performed on a chordal graph~\cite{Rose1976}.
\end{proof}

In order to prove the next lemmas of the main paper, we require
further auxiliary lemmas and definitions. A \emph{state}~$\mathcal{X}$
of the LBFS algorithm is a tuple
$(V(\mathcal{X}), \mathcal{S}(\mathcal{X}))$, where $V(\mathcal{X})$
is the sequence of already visited vertices at some point of the algorithm
and $\mathcal{S}(\mathcal{X})$ the current sequence of sets after the
last vertex in $V(\mathcal{X})$ has been processed.  From now
on, the term \emph{preceding neighbors} of $v$ at state $\mathcal{X}$
(denoted by $P_{\mathcal{X}}(v)$) describes the set
$N(v) \cap V(\mathcal{X})$, i.\,e., all neighbors of $v$ which were
visited before $v$ at state $\mathcal{X}$ of the LBFS.

\setcounter{new lemma}{11}

\begin{new lemma}
  \label{lemma:samepred}
  In a state $\mathcal{X}$ of the LBFS algorithm performed on a
  chordal graph $G$, two unvisited vertices $u$ and $v$ are in the same set if, and only if, the
  preceding neighbors of $u$ and $v$ are identical, i.\,e.,
  $P_{\mathcal{X}}(u) = P_{\mathcal{X}}(v)$. These preceding neighbors form a clique.
\end{new lemma}

\begin{proof}
  The second statement follows from the fact that the LBFS algorithm produces
  a valid PEO for chordal graphs.  Thus, when a vertex $w$ is visited, all preceding neighbors form
  a clique. This also holds for subsets of these neighbors and, hence,
  for every state $\mathcal{X}$.

  Let us now prove the first statement. First, observe that $P_{\mathcal{X}}(u) =
  P_{\mathcal{X}}(v)$ implies that $u$ and $v$ are in the same set. For each
  visited vertex, the sets are partitioned as described in
  line~\ref{line:partrefine} of Algorithm~\ref{alg:lbfs}. As this visited vertex is either a neighbor of
  $u$ and $v$ or of neither of them, these
  vertices are put in the same set.

  We show the other direction by contradiction. Assume $u$ and $v$ are
  in the same set, but the sets of preceding neighbors differ, i.\,e.,
  $P_{\mathcal{X}}(u) \neq P_{\mathcal{X}}(v)$. Moreover,
  assume without loss of generality that $x$ is the earliest visited vertex
  that is (i) in exactly one of these sets and that (ii) is a neighbor of $u$ but not $v$. When $x$
  is visited, $u$ and $v$ are in the same set by
  the argument above (the sets of preceding neighbors of $u$ and
  $v$ are identical in this state) and, according to the
  partitioning in line~\ref{line:partrefine}, $u$ and $v$ are put in different sets in
  $\mathcal{S}$. As no sets are joined in the LBFS, it follows
  that $u$ and $v$ stay in different sets and, thus, are not in the same set. A contradiction.
\end{proof}

Due to this result, we will also refer to the \emph{preceding
  neighbors of a set} $S \in \mathcal{S}(\mathcal{X})$ at state
$\mathcal{X}$ as $P_{\mathcal{X}}(S)$, which means the preceding
neighbors of any node in $S$.

\begin{new lemma}
  \label{lemma:superset}
  Let $\mathcal{X}$ be the state of an LBFS performed on a
  chordal graph~$G=(V,E)$, and let
  $S_1,S_2\in\mathcal{S}(\mathcal{X})$ such that $S_1$ precedes $S_2$
  in $\mathcal{S}(\mathcal{X})$
  and such that there are vertices 
  $u\in S_1$ and $v\in S_2$ with $\{u,v\}\in E$.
  Then $P_{\mathcal{X}}(S_1)\supsetneq P_{\mathcal{X}}(S_2)$.
\end{new lemma}

\begin{proof}
  Assume for the sake of contradiction that there is a vertex $x \in
  P_{\mathcal{X}}(S_2)\setminus P_{\mathcal{X}}(S_1)$. Let $u \in S_1$ and $v \in S_2$
  be the vertices with edge $u - v$.

  Vertex $v$ is visited after $u$ by the LBFS, as $S_2$ succeeds~$S_1$. Since $x$ was already visited as well, we have that the preceding neighbors of
  $v$ (when $v$ is visited) do not form a clique. This contradicts
  the fact that the LBFS ordering is the reverse of a PEO.
\end{proof}

\setcounter{lemma}{1}

\begin{claim lemma}
  Every AMO of a UCCG $G$ can be represented by an LBFS ordering.
\end{claim lemma}

\begin{proof}
  We show how the LBFS algorithm
  can be used to obtain a topological ordering that represents an AMO
  $\alpha$.

  The LBFS algorithm can freely choose a vertex from the first set. We
  restrict the algorithm to choose a vertex that has no
  incoming edges in $\alpha$ from an unvisited vertex. If such a vertex always
  exists, we obtain a topological ordering which represents $\alpha$.

  It is left to show that such a vertex indeed exists. We prove this
  by induction, where the base case is the source vertex of~$\alpha$,
  which can be chosen, as all vertices are in the first set.  Assume
  the first $k-1$ vertices were chosen from the (at that state) first
  set. We show that it is possible to pick a $k$th vertex from the
  first set, which has no incoming edges from an unvisited
  vertex.

  Let $\mathcal{X} = (V(\mathcal{X}), \mathcal{S}(\mathcal{X}))$ be
  the state when picking the $k$th vertex and assume,
  for the sake of contradiction, that there is no such vertex in the first
  set of $\mathcal{S}(\mathcal{X})$. This means, for all vertices $x$ in the
  first set $X$ of $\mathcal{S}(\mathcal{X})$, an unvisited vertex $z$
  exists with $x \leftarrow z$ in $\alpha$. There has to be at least one
  $z$ that is not in $X$, as otherwise there would be a cycle in $\alpha$. 
  This $z$ is in a later set than $x$, and there
  is an edge between $x$ and $z$. We deduce
  $P_{\mathcal{X}}(z)\subsetneq P_{\mathcal{X}}(x)$ with Lemma~\ref{lemma:superset}.

  Let $u \in P_{\mathcal{X}}(x)\setminus P_{\mathcal{X}}(z)$ be a
  preceding neighbor of $x$ but not $z$.
  The edge between $u$ and $v$ is correctly directed $u
  \rightarrow x$ by induction hypothesis. However, if $\alpha$ would contain the
  edge $x\leftarrow z$, then there would be a v-structure in $\alpha$,
  as there is no edge between $u$ and $z$. A contradiction.
\end{proof}

\begin{claim lemma}
  Every LBFS ordering starts with a maximal clique.
\end{claim lemma}

\begin{proof}
  At the beginning of the LBFS, an arbitrary vertex $u$ is chosen. In the
  subsequent steps, as long as there exists a vertex that is adjacent to
  all previously visited vertices, all vertices in the first set in $\mathcal{S}(\mathcal{X})$ have
  this property (in line~\ref{line:partrefine} of Algorithm~\ref{alg:lbfs} they are put in
  front sets of previously visited vertices). 
\end{proof}

\begin{claim corollary}
  Every AMO can be represented by a topological ordering which
  starts with a maximal clique.
\end{claim corollary}

\begin{proof}
  By Lemma~\ref{lemma:lbfsbegmc}, every LBFS ordering starts with
  a maximal clique. The statement follows from the fact that every AMO can be represented by an LBFS ordering
  (Lemma~\ref{lemma:allextlbfs}). 
\end{proof}

For the sake of readability, we repeat the following definition from the main paper:
\setcounter{definition}{0}
\begin{definition}
  Let $G$ be a UCCG over $V$, $K$ be a
  clique in $G$, and let $\pi(K)$ be a permutation of $K$.
  \begin{enumerate}
  \item The $\pi(K)$-orientation of $G$, also denoted $G^{\pi(K)}$,
    is the union of all AMOs of $G$ that can be represented by a
    topological ordering beginning with $\pi(K)$.
  \item Let $G^{K}$ denote the union of $\pi(K)$-orientations of $G$
    over all $\pi$, i.\,e.,\ let $G^{K} = \bigcup_{\pi} G^{\pi(K)}$.
  \item 
  Denote by $\chordalcomps_G(\pi(K))$ the undirected connected components 
  of $G^{\pi(K)}[V  \setminus K]$ and let 
  $\chordalcomps_G(K)$ denote the undirected connected components 
  of $G^{K}[V  \setminus K]$.
  \end{enumerate}
\end{definition}

We introduce a linear-time algorithm for finding the undirected
components of $G^{\pi(K)}$. Algorithm~\ref{alg:findsubproblems} proceeds as
Algorithm~\ref{alg:lbfs} in the main paper, with the 
exception that it makes sure that the vertices of the clique $K$ are
visited in order $\pi(K)$. In this way, the algorithm characterizes exactly the
$\pi(K)$-orientation, which makes the proofs cleaner compared to showing
the correctness of Algorithm~\ref{alg:lbfs} directly. Afterward, in Proposition~\ref{prop:cliqueid}, we
observe that indeed the permutation $\pi(K)$ does not influence orientations beyond the initial clique.
This will allow us to immediately conclude the correctness
of Algorithm~\ref{alg:lbfs} from the main paper (Theorem~\ref{thm:lintimesubp}).

\setcounter{algocf}{2}

\begin{algorithm}
  \caption{The algorithm computes $\chordalcomps_G(\pi(K))$.}
  \label{alg:findsubproblems}
  \DontPrintSemicolon
  \SetKwInOut{Input}{input}\SetKwInOut{Output}{output}
  \SetKwFor{Rep}{repeat}{}{end}
  \Input{A UCCG $G$, clique $K$, and permutation $\pi(K)$.}
  \Output{$\chordalcomps_G(\pi(K))$.}
  $\mathcal{S}\gets$ sequence of sets initialized with $(K,V\setminus K)$ \;
  $\tau\gets\text{empty list}$, $L\gets\emptyset$ \;
  \While{$\mathcal{S}$ is non-empty}{
    $X\gets\text{first non-empty set of $\mathcal{S}$}$ \;
    \uIf{$X \subseteq K$}{
      $\hbox to 0pt{$v$}\phantom{X}\gets\text{the first vertex of } X \text{ in } \pi(K)$ \;  \label{line:findvertex}}
    \Else{
      $\hbox to 0pt{$v$}\phantom{X}\gets\text{arbitrary vertex from $X$}$
    }
    Add vertex $v$ to the end of $\tau$. \;
    \If{$v$ is neither in a set in $L$ nor in $K$}{
      $L\gets L\cup \{X\}$ \; \label{line:settol}
      Output the undirected components of $G[X]$. \label{line:outsubs}
    }
    $X\gets X\setminus\{v\}$ \;
    Denote the current $\mathcal{S}$ by $(S_1, \dots, S_k)$. \;
    Replace each $S_i$  by $S_i \cap N(v), S_i
    \setminus N(v)$. \; \label{line:partrefine2}
    Remove all empty sets from $\mathcal{S}$. \; 
  }
\end{algorithm}

\begin{new lemma}
  \label{lemma:algpikcorrect}
  Algorithm~\ref{alg:findsubproblems} computes the undirected
  connected components of $G^{\pi(K)}$ in time
  $\mathcal{O}(|V|+|E|)$. Moreover, all remaining directed edges in $G^{\pi(K)}$
  are oriented as given by LBFS ordering $\tau$. 
\end{new lemma}

\begin{proof}  
  We first show that the edges between vertices in any set in $L$ are
  correctly left undirected, i.\,e., they are undirected in $G^{\pi(K)}$. Note that by definition of $G^{\pi(K)}$
  an edge $u - v$ is undirected, if there are AMOs represented by topological
  orderings that have $\pi(K)$ at the beginning and orient $u \rightarrow v$ and $u \leftarrow v$, respectively.

  Note that the algorithm starts
  with vertices in $K$ in the order $\pi(K)$. By
  Corollary~\ref{cor:lbfsordermao}, the LBFS ordering will always
  represent an AMO with these properties.
  
  Whenever, in a certain state $\mathcal{X}$ (after the initial clique
  $K$ has been visited), the algorithm chooses
  from the first set $X \in \mathcal{S}(\mathcal{X})$, it can choose the
  vertex arbitrarily. For any two neighbors $u$ and $v$ in $X$, there is an AMO with $u \rightarrow v$ (if
  we choose $u$ as first vertex) and one with $u \leftarrow v$ (if we
  choose $v$). Thus, the edge between $u$ and $v$ is undirected in
  $G^{\pi(K)}$. Hence, when a set $X$ is appended to $L$ in
  line~\ref{line:settol} of Algorithm~\ref{alg:findsubproblems}, edges in $G^{\pi(K)}[X]$ are undirected.

  We show that all remaining edges are oriented in the way
  given by the LBFS ordering $\tau$ in $G^{\pi(K)}$. 
  Clearly, the internal edges of $K$ are oriented correctly.
  We prove the correctness of the remaining edges by induction: For each vertex $u$, we show that the
  edge to every vertex $v$, which comes after it in the LBFS and is not in the same set
  in $L$, is directed as $u \rightarrow v$ in
  $G^{\pi(K)}$. This means that every AMO of $G$ whose
  topological ordering starts with $\pi(K)$ contains the edge $u \rightarrow v$.

  This holds for $K$, as edges to vertices that are not in~$K$ are
  always directed outwards from $K$.
  Assume the stated property holds for all previous vertices in the
  LBFS ordering $\tau$, we show
  it also holds for vertex $v$. Let $\mathcal{X}$ be the state of the
  LBFS before the first vertex from the set of $v$ in $L$ was
  visited. For any subsequent vertex $w$, which is not in
  the same set in $L$ as $v$, there is a preceding neighbor $u$ of $v$
  at state $\mathcal{X}$, which is not a preceding neighbor of $w$ at
  state $\mathcal{X}$. Otherwise
  $w$ would have been in the same set at state $\mathcal{X}$ by
  Lemma~\ref{lemma:samepred} and, thereby, be in the same set in $L$ as $v$. But
  then we have $u \rightarrow v - w$, with the correctness of edge $u
  \rightarrow v$ following from the induction hypothesis ($u$ cannot
  be in the same set in $L$ as $v$ by definition of $\mathcal{X}$). By the first
  Meek rule, it follows that $v \rightarrow w$ is in every
  AMO of $G^{\pi(K)}$.

  The sets in $L$ do not necessarily induce connected
  subgraphs. Thus, the algorithm returns the connected components of
  these sets, which are exactly the undirected connected components of
  $G^{\pi(k)}$.
  
  For the run time observe that the algorithm can be implemented in linear time with
  the same techniques used to implement the standard LBFS.
\end{proof}

We need further vocabulary to prove the next lemma.  Let
$H \in \chordalcomps_G(\pi(K))$ be returned by
Algorithm~\ref{alg:findsubproblems}. We denote the set of preceding
neighbors of the vertices in  $H$ at state $\mathcal{X}_{\text{out}}$,
the state when $H$ was output, by $P_{\text{out}}(H)$.
Note that this is a slight abuse of notation as $H$ is not
a set in $\mathcal{S}$ but a graph.

\begin{claim lemma}
  The undirected components in $\chordalcomps_G(\pi(K))$ are chordal and it holds that:
  \[
    \hamo(G^{\pi(K)}) = \prod_{H \in \chordalcomps_G(\pi(K))} \hamo(H).
  \]
\end{claim lemma}
\begin{proof}
  The chordality of the
  graphs in $\chordalcomps_G(\pi(K))$ follows from the
  correctness of Algorithm~\ref{alg:findsubproblems}
  (Lemma~\ref{lemma:algpikcorrect}) and the
  fact that the undirected components returned by
  Algorithm~\ref{alg:findsubproblems} are induced
  subgraphs of~$G$.
  
  The second part, i.\,e., the recursive formula, follows from the property that each component $H$ in
  $\chordalcomps_G(\pi(K))$ can be oriented independently of the
  directed part of $G^{\pi(K)}$. To prove this, we have to show that whenever a vertex $u$ is a
  parent of vertex $v \in V_H$, then $u$ is a parent of all vertices in
  $H$.  This fact was proven for CPDAGs as Lemma~10 by~\citet{He2008}
  and the property then follows analogously for the UCCGs of the
  $\pi(K)$-orientation as in Theorem~4 and Theorem~5 of the same paper.

  By Lemma~\ref{lemma:algpikcorrect}, the parents of a
  vertex $v \in V_H$ in the $\pi(K)$-orientation are
  the vertices in $P_{\text{out}}(H)$. All vertices in this
  set are neighbors of $v$, considered before $v$ and not in the same
  set in $L$. All other neighbors of $v$ are visited later or are in
  the same set in $L$ as $v$ and can, thus, not be parents. The set
  $P_{\text{out}}(H)$ is by definition the same for each vertex in
  $H$. 
\end{proof}

\begin{claim proposition}
  Let $G$ be a UCCG and $K$ be a clique of $G$.
  For each permutation $\pi(K)$ it is true that all edges of $G^{\pi(K)}$
  coincide with the edges of $G^K$, excluding the edges connecting 
  the vertices in $K$. 
  Hence, $\cC_G(\pi(K)) = \cC_G(K)$ and it holds that:
  \[
    \sum_{\text{$\pi$ over $K$}} \hamo(G^{\pi(K)}) = |K|! \times \prod_{\balap{H \in \chordalcomps_G(K)}} \hamo(H).
  \]
\end{claim proposition}

\begin{proof}
  We prove the statement by showing that, for two arbitrary permutations $\pi(K)$ and
  $\pi'(K)$, the edges in $G^{\pi(K)}$ and $G^{\pi'(K)}$ coincide,
  excluding the edges connecting the vertices in $K$.

  The graph
  $G^{\pi(K)}$ is defined as the union of all AMOs, which can be
  represented by a topological ordering starting with $\pi(K)$. Take
  such an AMO $\alpha$ and, in a corresponding topological ordering $\tau$, replace
  $\pi(K)$ by $\pi'(K)$ obtaining a new topological ordering $\tau'$.
  The orientation $\alpha'$ represented by $\tau'$ is, by definition, acyclic and,
  moreover, moral. For the latter property, assume for a
  contradiction, that there is a v-structure (immorality) $a
  \rightarrow b \leftarrow c$. Because $\alpha$ is moral and only edge
  directions internal in $K$ have been changed in $\alpha'$, it has to
  hold that either
  \begin{enumerate}
  \item two vertices of $a,b,c$ are in $K$ (w.l.o.g.\ assume these are $a$ and
    $b$), but then we have $b \rightarrow c \not\in K$ and not $b
    \leftarrow c$, or
  \item all three vertices are in $K$, but then $a \rightarrow b
    \leftarrow c$ is no induced subgraph.
  \end{enumerate}
  Hence, such a v-structure can not exist and $\alpha'$ is moral as
  well. The reverse direction follows equivalently.

  Therefore, the union of all AMOs, which can be represented by a
  topological ordering $\tau'$ starting with $\pi'(K)$, yields the exact
  same graph as for $G^{\pi(K)}$, excluding the internal edges in $K$. Thus,
  $\chordalcomps_G(\pi(K)) = \chordalcomps_G(\pi'(K))$ and, by definition,
  $\chordalcomps_G(\pi(K)) = \chordalcomps_G(K)$. The recursive formula is immediately implied by
  this fact and Lemma~\ref{lemma:independentOrientation}.
 \end{proof}
 
 \begin{claim theorem}
  For a chordal graph $G$ and a clique $K$, Algorithm~\ref{alg:lbfs}
  computes $\chordalcomps_G(K)$ in time
  $\mathcal{O}(|V| + |E|)$. 
\end{claim theorem}

\begin{proof}
  By Lemma~\ref{lemma:algpikcorrect},
  Algorithm~\ref{alg:findsubproblems} correctly computes the UCCGs of
  the $\pi(K)$-orientation. As these UCCGs are identical for each
  orientation $\pi(K)$ and
  Algorithm~\ref{alg:lbfs} proceeds just as
  Algorithm~\ref{alg:findsubproblems}, with the only difference that the
  former makes no restriction on the order the vertices in $K$ are
  visited (thus visiting them in arbitrary permutation $\pi'(K)$), the
  correctness follows.
\end{proof}

\section*{Proofs of Section 4:\\[.5ex] Counting MA-Orientations with\\ Minimal Separators and Maximal Cliques}
\begin{claim lemma}
  Let $\alpha$ be an AMO of a chordal graph~$G$ and $\tau_1$,
  $\tau_2$ be two topological orderings that represent $\alpha$. Then
  $\tau_1$ and $\tau_2$ have a common prefix $S\in\separators(G)\cup\cliques(G)$.
\end{claim lemma}
\begin{proof}
  Assume by Corollary~\ref{cor:startclique} that $\tau_1$ starts with the maximal clique $K_1$ and $\tau_2$ with
  the maximal clique $K_2$. Since every AMO of a UCCG has a
  unique source, $\tau_1$ and $\tau_2$ start with the same vertex and,
  hence, $K_1\cap K_2=S\neq\emptyset$.

  We first show that $\tau_1$ and $\tau_2$ have to start with
  $S$. Assume for a contradiction that in $\tau_1$ there is a vertex
  $u \not\in S$ before a $v \in S$. The edge between $u$ and $v$ is
  directed as $u \rightarrow v$ in $\alpha$, but as $v \in K_2$ and
  $u \not\in K_2$, the ordering $\tau_2$ implies $u \leftarrow v$.

  If $K_1=K_2$ then $S\in\cliques(G)$ and we are done. We prove that
  otherwise $S$ is a minimal separator in $G$
  that separates $P_1=K_1\setminus S$ from $P_2=K_2\setminus S$. Note
  that the minimality follows by definition. It remains to show that
  $S$ indeed separates $P_1$ and $P_2$. For a contradiction, let
  $x_1 \in P_1 - x_2 - \dots - x_{k-1} - x_k \in P_2$ be a shortest
  $P_1$-$P_2$-path in $G[V\setminus S]$ with $x_i\not\in K_1\cup K_2$
  for $i \in \{2, \dots, k-1\}$.
  According to $\tau_1$, we have the edge $x_1 \rightarrow x_2$ in
  $\alpha$. Since we consider a shortest path,
  $x_{i-1} - x_{i} - x_{i+1}$ is always an induced subgraph and, thus,
  an iterative application of the first Meek rule implies
  $x_{k-1} \rightarrow x_k$.  However, $\tau_2$ would imply the edge
  $x_{k-1} \leftarrow x_k$ in $\alpha$. A contradiction.
\end{proof}

We repeat the following definition from the main text:
\begin{definition}
For a set $S$ and a collection $R$ of subsets of~$S$, 
we define $\phi(S,R)$ as the number of
all permutations of $S$ that do not have 
a set $S'\in R$ as prefix. 
\end{definition}

\begin{claim proposition}
  Let $G$ be a UCCG. Then:
  \[\hamo(G)=\sum_{S\in\Delta(G)\cup\Pi(G)}
  \phi(S)
  \times 
  \prod_{\balap{H\in\chordalcomps_G(S)}}\hamo(H).
  \]
\end{claim proposition}
\begin{proof}
  By the choice of $S$ and the definition of $\chordalcomps_G(S)$,
  everything counted by the formula is a topological ordering
  representing an AMO.  We argue that every AMO~$\alpha$ is counted
  exactly once. Let $S\in\separators(G)\cup\cliques(G)$ be the
  smallest common prefix of all topological orderings in
  $\mathrm{top}(\alpha)$~--~which is well-defined by
  Lemma~\ref{lemma:minimalSeparator}. First observe that, by the
  minimality of $S$, $\alpha$ is counted at the term for $S$: There is
  no other prefix $\tilde S\subsetneq S$ of the topological orderings
  with $\tilde S\in\separators(G)$ and $\tilde S\in\phi(S)$.

  On the other hand, $S$ is the only term in the sum at which we can
  count $\alpha$, as for any larger $\tilde S$ with $S\subsetneq \tilde
  S$ that is a prefix of some $\tau\in\mathrm{top}(\alpha)$, we have
  $S$ is considered in $\phi(\tilde S)$.
\end{proof}

\section*{Proofs of Section 5:\\ The Clique-Picking Algorithm}
\begin{claim lemma}
  Let $S$ be a set and $R=\{X_1,\dots,X_{\ell}\}$ be a collection of
  subsets of $S$ with $X_1\subsetneq X_2\subsetneq\dots\subsetneq
  X_{\ell}$. Then:
  \[
    \phi(S,R) = |S|!
    -\sum_{i=1}^{\ell}|S\setminus X_i|!\cdot\phi(X_i,\{X_1,\dots,X_{i-1}\}).    
  \]
\end{claim lemma}

\begin{proof}
  We prove the statement by induction over $\ell$ with the base case
  $\phi(S,\emptyset)=|S|!$. Consider a set $S$ and a collection
  $R=\{X_1,\dots,X_{\ell}\}$ of subsets of $S$. We can compute
  $\phi(S,R)$ by taking $\phi(S,\{X_1,\dots,X_{\ell-1}\})$ (the number
  of permutations of $S$ that do not start with
  $X_1,\dots,X_{\ell-1}$) and by subtracting the number of
  permutations that start with $X_{\ell}$ but none of the other $X_i$, i.\,e.,
  \begin{align*}
    \phi(S,R) &= \phi(S,\{X_1,\dots,X_{\ell-1}\})\\
    &\qquad - |S\setminus X_{\ell}|!\cdot\phi(X_{\ell},\{X_1,\dots,X_{\ell-1}\}).
  \end{align*}
  Inserting the induction hypothesis, we obtain:
  \begin{align*}
    \phi(S,R) &= |S|! -\sum_{i=1}^{\ell-1}|S\setminus X_i|!\cdot\phi(X_i,\{X_1,\dots,X_{i-1}\})\\
     &\qquad - |S\setminus
       X_{\ell}|!\cdot\phi(X_{\ell},\{X_1,\dots,X_{\ell-1}\})\\
    &=|S|!-\sum_{i=1}^{\ell}|S\setminus X_i|!\cdot\phi(X_i,\{X_1,\dots,X_{i-1}\}).
  \end{align*}
\end{proof}

Recall, for the following proof, the definition of a forbidden prefix:
\begin{definition}
  Let $G$ be a UCCG, $\mathcal{T} = (T, r, \iota)$ a rooted clique tree of
  $G$, $v$ a node in $T$ and $r = x_1, x_2, \dots, x_p = v$ the unique
  $r$-$v$-path. We define the set $\mathrm{FP}(v, \mathcal{T})$ to contain
  all sets $\iota(x_i) \cap \iota(x_{i+1}) \subseteq \iota
  (v)$ for $1 \leq i < p$. 
\end{definition}

\begin{claim lemma}
  We can order the elements of the set $\mathrm{FP}(v, \mathcal{T})$ as $X_1
  \subsetneq X_2 \subsetneq \dots \subsetneq X_\ell$.
\end{claim lemma}

\begin{proof}
  The ordering of the sets is given by the natural order along the
  path from the root $r$ to node $v$.
  The sets in $\mathrm{FP}(v, \mathcal{T})$
  satisfy $\iota(x_i)\cap \iota(x_{i+1})\subseteq \iota(v)$. By the
  definition of a clique tree, we have $\iota(x_i)\cap \iota(x_{i+1})\subseteq \iota(y)$ for
  each $y$ that lies on the $x_i$-$v$-path in $T$. Hence, each such $y$
  can only add supersets of $\iota(x_i) \cap \iota(x_{i+1})$ to $\mathrm{FP}(v, \mathcal{T})$.
\end{proof}

Recall the definition of $S$-flowers and bouquets:
\begin{definition} 
  An \emph{$S$-flower} for a minimal separator~$S$ is a maximal set
  $F\subseteq\{\,K\mid K\in\cliques(G)\wedge S\subseteq K\,\}$ such that $\bigcup_{K\in F}K$ is
  connected in $G[V\setminus S]$. The \emph{bouquet}~$\bouquet(S)$ of a
  minimal separator $S$ is the set of all $S$-flowers.
\end{definition}

\begin{claim lemma}
  An $S$-flower $F$ is a connected subtree in a rooted clique tree $(T,r,\iota)$.
\end{claim lemma}
\begin{proof}
  Assume for a contradiction that $F$ is not
  connected in $T$. Then there are cliques $K_1,K_2\in F$ that
  are connected by the unique path $K_1-\tilde K-\dots-K_2$ with $\tilde
  K\not\in F$. Since $\iota^{-1}(S)$ is connected, we
  have $S\subseteq\tilde K$. By the maximality of~$F$, we
  have $K_1\cap\tilde K=S$. But then $S$ separates $K_1\setminus S$
  from $K_2\setminus S$, which contradicts the definition of $S$-flowers.
\end{proof}

\begin{claim lemma}
  For any minimal separator $S$, the bouquet $\bouquet(S)$ is a
  partition of $\iota^{-1}(S)$.
\end{claim lemma}
\begin{proof}
  For each $x\in\iota^{-1}(S)$, the maximal clique $\iota(x)$ is in
  some $S$-flower by definition. However, no maximal clique can be
  in two $S$-flowers, as these flowers would then be in the same connected
  component in $G[V\setminus S]$.
\end{proof}

\begin{claim lemma}
 There is a unique least $S$-flower in $\bouquet(S)$ with respect to $\prec_T$.
\end{claim lemma}
\begin{proof}
  Assume, there is no unique least $S$-flower.
  Then there are two minimal $S$-flowers which are
  incomparable. However, by
  Lemma~\ref{lemma:connflowers} and~\ref{lemma:bouquetpart}, and the definition
  of the partial order, there has to be another $S$-flower closer to
  the root and, thus, lesser given the partial order. A contradiction.
\end{proof}

\begin{claim lemma}
  Let $\alpha$ be an AMO such that every topological ordering that
  represents $\alpha$ has the minimal separator~$S$ as prefix. Then
  every $F\in\bouquet(S)$ contains a clique~$K$ such that there is a
  $\tau\in\mathrm{top}(\alpha)$ starting with $K$.
\end{claim lemma}
\begin{proof}
  Let $\tau$ be a topological ordering representing $\alpha$ that starts with $S$.
  By Proposition~\ref{cor:startclique}, there is at least one clique $K$ with
  $S\subseteq K$ such that $\tau$ has the form $\tau=(S,K\setminus S,V\setminus K)$. Let $F\in\bouquet(S)$ be
  the flower containing $K$ and $F'\neq F$ be another $S$-flower with
  some $K'\in F'$. Observe that $K\setminus S$ is disconnected from
  $K'\setminus S$ in $G[V\setminus S]$. Therefore, there is a
  topological ordering of the form $(S,K'\setminus S, V\setminus K')$
  that represents $\alpha$ as well.
\end{proof}
\begin{claim proposition}
  Let $G$ be a UCCG and $\mathcal{T} = (T, r, \iota)$ be a rooted
  clique tree of $G$. Then:
  \[
    \hamo(G) = \sum_{\balap{v\in V_T}} \phi(\iota(v), \mathrm{FP}(v, \mathcal{T})) 
    \times \prod_{\balap{H \in \chordalcomps_G(\iota(v))}} \hamo(H).
  \]
\end{claim proposition}

\begin{proof}
  We have to show that every
  AMO~$\alpha$ is counted exactly once. Recall that
  $\mathrm{top}(\alpha)=\{\tau_1,\dots,\tau_{\ell}\}$ is the set of topological orderings
  that represent $\alpha$, and that the rooted clique tree $(T,r,\iota)$
  implies a partial order $\prec_T$ on flowers, which in return
  defines partial order $\prec_{\alpha}$ on the set of maximal cliques that are at the
  beginning of some $\tau\in\mathrm{top}(\alpha)$.

  \begin{claim}
    There is a unique least maximal clique $K\in\cliques(G)$ with respect to $\prec_{\alpha}$.
  \end{claim}
  \begin{proof}
    Let $\mathrm{top}'(\alpha)\subseteq\mathrm{top}(\alpha)$ be an arbitrary subset of the
    topological orderings that represent~$\alpha$ and let $\mu$ be the
    number of different maximal cliques with which elements in
    $\mathrm{top}'(\alpha)$ start. We prove the claim by induction over
    $\mu$. In the base case, all elements in $\mathrm{top}'(\alpha)$ start with
    the same set $S\in\cliques(G)$ and, of course, this is the unique
    least maximal clique. For $\mu>1$ we observe
    that, by Lemma~\ref{lemma:minimalSeparator}, all
    $\tau\in\mathrm{top}'(\alpha)$ start with the same $S\in\separators(G)$.

    Consider the bouquet $\bouquet(S)$, which is partially ordered
    by~$\prec_T$. Lemma~\ref{lemma:flowersAreOrdered} states that there is a
    unique least $S$-flower $F\in\bouquet(S)$ with respect to
    $\prec_T$, and by the definition of~$\prec_{\alpha}$ the maximal
    cliques occurring in $F$ precede the others. Therefore, we reduce
    $\mathrm{top}'(\alpha)$ to the set $\mathrm{top}''(\alpha)$ of topological orderings that
    start with a maximal clique in $F$. This set is non-empty by
    Lemma~\ref{lemma:BSeveryAMO} and contains, by the induction
    hypothesis, a unique least maximal clique.
  \end{proof}

  We complete the proof by showing that
  the formula counts~$\alpha$ at the term for the unique least
  maximal clique~$K$ from the previous claim. To see this, we need to
  prove that (i) $\alpha$ can be counted at the clique $K$ (i.\,e., there is no
  set $S\in \mathrm{FP}(\iota^{-1}(K), \mathcal{T})$ preventing $\alpha$
  from being counted), and (ii) that
  $\alpha$ is not counted somewhere else (i.\,e., there is some $S\in
  \mathrm{FP}(\iota^{-1}(K'), \mathcal{T})$ for all other $K'\in\cliques(G)$ that can be at the beginning
  of some $\tau\in\mathrm{top}(\alpha)$).

  \begin{claim}
    Let $\alpha$ be an AMO and $K\in\cliques(G)$ be the least
    maximal clique (with respect to $\prec_{\alpha}$) that is a prefix of
    some $\tau\in\mathrm{top}(\alpha)$. Then there is no $S\in\Delta(G)$ with
    $S\in \mathrm{FP}(\iota^{-1}(K), \mathcal{T})$ that is a prefix of $\tau$.
  \end{claim}
  \begin{proof}
    Assume for a contradiction that there would be such a
    $S\in\Delta(G)$ and let $F\in\bouquet(S)$ be the $S$-flower
    containing~$K$. Since $S\in \mathrm{FP}(\iota^{-1}(K), \mathcal{T})$, there is another flower
    $F'\in\bouquet(S)$ with $F'\prec_T
    F$. Lemma~\ref{lemma:BSeveryAMO} tells us that there is another
    clique $K'\in F'$ that is at the beginning of some
    $\tau'\in\mathrm{top}(\alpha)$. However, then we have
    $K'\prec_{\alpha}K$~--~contradicting the minimality of $K$.
  \end{proof}
  \begin{claim}
    Let $\tau_1,\tau_2\in\mathrm{top}(\alpha)$ be two topological orderings
    starting with $K_1,K_2\in\cliques(G)$, respectively. If
    $K_1\prec_{\alpha} K_2$ then $K_1\cap K_2=S\in
    \mathrm{FP}(\iota^{-1}(K_2), \mathcal{T})$.
  \end{claim}
  \begin{proof}
    Since $K_1$ and $K_2$ correspond to $\tau_1,\tau_2\in\mathrm{top}(\alpha)$,
    we have $K_1\cap K_2=S\in\cliques(G)\cup\separators(G)$ by
    Lemma~\ref{lemma:minimalSeparator}~--~in fact, $S$ is a prefix of
    $\tau_1$ and $\tau_2$. As we assume $K_1\prec_{\alpha} K_2$, we have
    $K_1\neq K_2$ and, thus, $S\in\Delta(G)$. Let
    $F_1,F_2\in\bouquet(S)$ be the $S$-flowers containing $K_1$ and
    $K_2$, respectively. The order $K_1\prec_{\alpha}K_2$ implies
    $F_1\prec_T F_2$ (item (iii) in the definition of~$\prec_{\alpha}$), meaning that $F_1$ contains some node of
    $(T,r,\iota)$ that is on the unique path from $F_2$ to the root of
    $T$. But then, by the definition of $S$-flowers and
    Lemma~\ref{lemma:bouquetpart}, the first edge on this path that
    leads to a node in $F_1$ connects two nodes $x,y$ with
    $\iota(x)\cap\iota(y)=S$. Hence, $S\in \mathrm{FP}(\iota^{-1}(K_2), \mathcal{T})$.
  \end{proof}
  This completes the proof of Proposition~\ref{proposition:fpFormula}.
\end{proof}

\begin{claim theorem}
  For an input UCCG $G$, Algorithm~\ref{alg:cliquepicking} returns
  the number of AMOs of $G$.
\end{claim theorem}
\begin{proof}
  Observe that recursive calls are performed in
  line~\ref{line:mult} if  $\chordalcomps_G(\iota(v))\neq\emptyset$. The only graphs with
  $\chordalcomps_G(S)=\emptyset$ for all $S\in\cliques(G)$ are the
  complete graphs, i.\,e., the graphs with $|\Pi(G)|=1$.
  We have $|\cliques(H)|<|\cliques(G)|$ for all graphs $G$ and
  $H=G[V\setminus S]$ with $S\in\cliques(G)$. Hence, we may assume by
  induction over $|\cliques(G)|$ that the subproblems are handled
  correctly~--~the base case being given by complete graphs.

  The correctness of the algorithm follows from
  Proposition~\ref{proposition:fpFormula}, as it traverses the clique
  tree with a BFS in order to compute the sets
  $\mathrm{FP}(v, \mathcal{T})$ and evaluate this formula. 
\end{proof}

\section*{Proofs of Section 6:\\ The Complexity of \#\kern-0.5ptAMO}
Recall that for $H \in \chordalcomps_G(\pi(K))$, we defined
$P_{\text{out}}(H)$ to be the set of preceding neighbors of the
vertices of $H$ at state $\mathcal{X}_{\text{out}}$ when $H$ was
output during Algorithm~\ref{alg:findsubproblems}. Now, we consider $H \in
\chordalcomps_G(K)$. We use the same notation to refer to the preceding neighbors in the
analogously defined state in Algorithm~\ref{alg:lbfs}. In fact, as
$P_{\text{out}}(H)$ is independent of the permutation $\pi(K)$ (Proposition~\ref{prop:cliqueid}), it will
be the same set.

We denote the previously visited vertices at
$\mathcal{X}_{\text{out}}$, which are not in $P_{\text{out}}(H)$, by $W$, i.\,e.,
$W = V(\mathcal{X}_{\text{out}}) \setminus P_{\text{out}}(H)$.

\begin{new lemma}
  \label{lemma:spminsep}
  Let $G$ be a chordal graph, $K \in \cliques(G)$, and $H \in
  \chordalcomps_G(K)$. Then, $P_{\text{out}}(H)$ separates~$V_H$ from $W =
  V(\mathcal{X}_{\text{out}}) \setminus P_{\text{out}}(H)$ and is a
  minimal separator of $G$.
\end{new lemma}

\begin{proof}
  The set $P_{\text{out}}(H)$ is a proper subset of all previously
  visited vertices ($V_H$ is not part of the \emph{maximal} clique $K$ Algorithm~\ref{alg:lbfs}
  starts with). Since $P_{\text{out}}(H)$ contains all visited
  neighbors of $V_H$, it separates $V_H$ from $W$. To see this, assume for
  sake of contradiction that there is a path from $v \in V_H$ to $w \in
  W$  in $G$ without a vertex in $P_{\text{out}}(H)$. Consider the shortest such path and let
  $y$ be the first vertex with successor $z$ preceding it in
  the LBFS ordering produced by Algorithm~\ref{alg:lbfs}: $v - \dots - x - y - z - \dots - w$. Then
  $\{x,z\}\in E_G$, as  the LBFS computes a PEO. Hence, the path is
  not the shortest path and, thus, $y$ cannot
  exist. Since there can be no direct edge from $v$ to $w$, the set
  $P_{\text{out}}$ is indeed a separator.

  We prove that there is a vertex in $W$, which is a neighbor of
  all vertices in $P_{\text{out}}(H)$. Consider the vertex in
  $P_{\text{out}}(H)$, which is visited last (denoted by $p$). When
  vertex $p$ is processed, it has to have a neighbor $x \in W$, which was
  previously visited, else $p$ would be part of $H$. This is because the
  preceding neighbors would be identical to the ones of the vertices
  in $H$ (i.\,e., $\mathcal{P}_{\text{out}}(H) \setminus \{p\}$),
  meaning by Lemma~\ref{lemma:samepred} that $p$ would be in the same
  set in $\mathcal{S}$. It would follow that
  either $p$ and the vertices in $H$ are appended to $L$ when $p$ is
  visited or were already appended to $L$ previously. In both cases, $p$ would
  be in $V_H$, which is a contradiction.

  Hence, such a vertex $x$ has to exist. Moreover, $x$
  has to be connected to all vertices in $P_{\text{out}}(H)$ because
  of the PEO property (all preceding neighbors of a vertex form a
  clique).

  From the first part of the proof, we know that $x$ and $y \in H$ are
  separated by $P_{\text{out}}(H)$. As both $x$ and $y$ are fully
  connected to $P_{\text{out}}(H)$, it follows that this set is also a \emph{minimal} $x-y$ separator.

\end{proof}

\begin{new lemma} \label{lemma:flspbijection}
  Let $G$ be a chordal graph for which the number of AMOs is computed
  with the function $\texttt{count}$ in
  Algorithm~\ref{alg:cliquepicking}. Let $H$ be any chordal graph for
  which \texttt{count} is called in the recursion (for $H \neq G$).
  Then $V_H = F \setminus S$ for some $S$-flower $F$ in $G$ with $S \in \separators(G)$.
\end{new lemma}

\begin{proof}
  Let $S_H$ be the union of all sets $P_{\text{out}}(\tilde{G})$ for
  $\tilde{G}$ on the recursive call stack from the input graph $G$ to
  currently considered subgraph $H$. We
  define $P_{\text{out}}(G) = \emptyset$ for convenience.

  Let $H\neq G$, we show by induction that (i) $S_H$ is a minimal
  separator in $G$, (ii) $S_H$ is fully connected to $V_H$ in $G$, and 
  (iii)~$V_H = F \setminus S_H$ for some $S_H$-flower $F$.

  In the base case, $H \in \chordalcomps_G(K)$ for some clique $K \in \cliques(G)$.  By Lemma~\ref{lemma:spminsep},
  $S_H$ is a minimal separator in $G$, which is by definition
  connected to all vertices in $H$. Hence, as $H$ is connected, $V_H
  \subseteq F \setminus S_H$ holds for an $S_H$-flower $F$. We show the equality by
  contradiction. Assume there is a vertex $v \in F \setminus S$ but
  not in $V_H$. Then $v$ can neither be a vertex in $W$ nor the neighbor
  of a vertex in $W$, as by the definition of
  flowers there has to be a path from $v$ to $V_H$ in $G[V_G
  \setminus S_H]$~--~this would violate that $V_H$ is separated from
  $W$ by $S_H$ (Lemma~\ref{lemma:spminsep}). Moreover, $v$ is a neighbor of all vertices in $S_H$.
  Hence, we have $P_{\text{out}}(v) = P_{\text{out}}(H) = S_H$ and $v \in V_H$. A
  contradiction.

  Assume \texttt{count} is called with a graph $H \in
  \chordalcomps_{G'}(K)$ for some graph $G'$ and $K\in\cliques(G')$. By induction hypothesis,
  we have that $S_{G'}$ is a minimal separator in $G$ and fully
  connected to $V_{G'}$. Moreover, $V_{G'} = F' \setminus S_{G'}$ for some
  $F'$-flower of $S_{G'}$. Now, $P_{\text{out}}(H)$ is by
  Lemma~\ref{lemma:spminsep} a minimal separator in $G'$ for some vertices $x$
  and $y$. As $x$ and $y$ are connected to every vertex in $S_{G'}$,
  it follows that $S_H = S_{G'} \cup P_{\text{out}}(H)$ is a minimal
  $x$-$y$ separator in $G$. Furthermore, $S_H$ is fully connected to
  $V_H$ in $G$ and it can be easily seen that $V_H \subseteq F \setminus S_H$. To
  show equality, observe that every vertex $v$ in $F \setminus S_H$
  is in $V_{G'}$ (if it is not separated from $V_H$ by $S_H$ in $G$, it is
  clearly not separated from $V_H$ by $S_{G'}$ in $G$, hence $v$ is in $F' \setminus S_{G'} = V_{G'}$). Thus, the same argument
  as in the base case applies and the statement follows.
\end{proof}

\begin{claim proposition}
  Let $G$ be a UCCG. The number of distinct UCCGs explored by \texttt{count} is bounded by $2|\cliques(G)|-1$. 
\end{claim proposition}

\begin{proof}
  By Lemma~\ref{lemma:flspbijection}, it remains to bound the number
  of flowers in $G$. Each flower is associated with a minimal
  separator $S$ and there are at most $|\cliques(G)|-1$ such
  separators, as they are associated with the edges of the clique
  tree~\cite{Blair1993}.  Let $r$ (which is initially
  $|\cliques(G)|-1$) be an upper bound for the number of remaining
  separators. Now consider separator $S$. If $\bouquet(S)$ has $k$
  flowers, $S$ can be found on at least $k-1$ edges of the clique
  tree, namely the edges between the flowers (by
  Proposition~\ref{lemma:bouquetpart} the flowers partition the
  bouquet and by definition of flowers, the intersection of cliques
  from two $S$-flowers has to be a subset of $S$). Thus, we have at
  most $r - (k-1)$ remaining separators. The maximum number of flowers
  is obtained when the quotient $k / (k-1)$ is maximal. This is the
  case for $k = 2$. It follows that there are at most
  $2(|\cliques(G)| - 1)$ flowers.

  When bounding the number of explored UCCGs, we additionally take
  into account the input graph and obtain as bound
  $2(|\cliques(G)| - 1) + 1 = 2|\cliques(G)| - 1$.
\end{proof}

\begin{claim theorem}
  The Clique-Picking algorithm runs in time
  $\mathcal{O}\big(\,|\cliques(G)|^2 \cdot (|V| + |E|)\,\big)$.
\end{claim theorem}

\begin{proof}
  By Proposition~\ref{prop:subpbound}, \texttt{count} explores
  $\mathcal{O}(|\cliques(G)|)$ distinct UCCGs. For each of them, the clique tree is computed in
  time $\mathcal{O}(|V| + |E|)$. Afterwards, for each maximal clique,
  the subproblems are computed by Algorithm~\ref{alg:lbfs} in time
  $\mathcal{O}(|V| + |E|)$ by Theorem~\ref{thm:lintimesubp}.

  For the computation of $\phi(S,\mathrm{FP}(v,\mathcal{T}))$, note
  that $\mathrm{FP}$ can be
  computed straightforwardly: Traverse the clique tree with a
  BFS, keep track of the nodes on the path from root $r$ to any
  visited node, compute $\mathrm{FP}$ with its definition.

  The function $\phi$ can be evaluated with the formula from
  Lemma~\ref{lemma:efficientPhi}. There are $\mathcal{O}(|S|)$
  subproblems and for each a sum over $\mathcal{O}(|S|)$ terms has to
  be computed (as $l$ is always smaller than $|S|$). Because $S$ is a
  clique, the effort is in $\mathcal{O}(|E|)$. 
\end{proof}

\begin{algorithm}
    \caption{The algorithm returns a uniform permutation of clique $K$ under the constraint that the permutation does not start with a set in $\mathrm{FP}$.}
  \label{alg:drawperm}
  \SetKwInOut{Input}{input}\SetKwInOut{Output}{output}
  \DontPrintSemicolon
  \Input{Clique $K$, list $\mathrm{FP} = (X_1, X_2, \dots, X_\ell)$ with
  $X_1 \subsetneq X_2 \subsetneq \dots \subsetneq X_\ell$.}
  \Output{Permutation $\pi$.}
  \SetKwFunction{FPerm}{drawperm}
  \SetKwFunction{FVertex}{drawvertex}
  \SetKwFunction{FConcat}{concat}
  \SetKwProg{Fn}{function}{}{end}
  \Fn{\FPerm{$K$, $\mathrm{FP}$, $\mathrm{memo}$}}{
    \ForEach{$v \in K$}{
      \uIf{$v$ in $X_\ell$}{
        Let $i$ be smallest such that $v \in X_i$. \;
        $\mathrm{wt}(v) \gets \phi(K \setminus v, \{X_i \setminus v, \dots, X_\ell \setminus v\})$ \;
      }
      \Else{
        $\mathrm{wt}(v) \gets \phi(K \setminus v, \emptyset)$ \;
      }
    }
    $v \leftarrow \FVertex(\mathrm{wt})$ \;
    \uIf{$v$ in $X_\ell$}{
      Let $i$ be smallest such that $v \in X_i$. \;
      $\tilde{\mathrm{FP}} \gets \{X_i \setminus v, \dots, X_\ell \setminus v\}$ \;
    }
    \Else{
      $\tilde{\mathrm{FP}} \gets \emptyset$ \;
    }
    \uIf{$|K \setminus v| = 0$}{
      \KwRet $(v)$ \;
    }
    \Else{
      \KwRet $\FConcat(v, \FPerm(K \setminus v, \tilde{\mathrm{FP}}))$ \; 
    }
  }
\end{algorithm} 

With $\spe(K, \mathrm{FP})$ we denote the set of permutations
of clique $K$ without a prefix in $\mathrm{FP}$ and with $\pe(K,
\mathrm{FP})$ its size (which coincides with $\phi(K, \mathrm{FP})$). 

\begin{new lemma} \label{lemma:drawpermcorrect}
  Algorithm~\ref{alg:drawperm} samples a permutation $\pi \in \spe(K,
  \mathrm{FP})$ uniformly at random.
\end{new lemma}

\begin{proof}
  We show this by induction over the size of the clique~$K$. In the
  base case, we have a singleton $K = \{v\}$ and $\pi = (v)$ will be chosen with
  probability $1 = 1 / \pe(K, \mathrm{FP})$.

  Let $K$ be a clique with $|K| > 1$. Then the algorithm will
  choose a vertex $v$ and, recursively, find a uniform
  permutation for the remaining clique
  $\tilde{K} = K \setminus v$ (regarding new forbidden prefixes $\tilde{\mathrm{FP}}$). This 
  $\tilde{\mathrm{FP}}$ contains only sets $X_i, \dots, X_\ell\in\mathrm{FP}$ 
  that contain $v$, as the other forbidden prefixes cannot occur with $v$
  being picked first. Hence,
  \[
    \tilde{\mathrm{FP}} = \{X_i \setminus v, \dots, X_\ell \setminus v\}
  \]
  with $i$ being the smallest index such that $v \in X_i$.
  For the probability $\mathrm{Pr}(v_\pi \; | \; K, \mathrm{FP})$ that the first vertex $v$ is chosen
  according to a permutation $\pi$ (given some $K$ and $\mathrm{FP}$), it holds that:
  \[
    \mathrm{Pr}(v_\pi \; | \; K, \mathrm{FP}) = \frac{\pe(K \setminus v_\pi,
      \tilde{\mathrm{FP}})}{\pe(K, \mathrm{FP})}.
  \]
  This is due to the fact that \texttt{drawvertex} samples $v$
  proportional to the weights $\mathrm{wt}(v)$, which is the number
  of permutations without forbidden prefix in $\tilde{\mathrm{FP}}$
  that start with $v$.

  Now, we want to compute $\mathrm{Pr}(\pi \; | \; K, \mathrm{FP})$,
  i.\,e., the probability that the vertices in $K$ with forbidden
  prefixes $\mathrm{FP}$ are permuted according to (some not
  forbidden) $\pi$. For this, we make use of the induction hypothesis:
  \[
    \mathrm{Pr}(\tilde{\pi} \; | \; K \setminus v_\pi, \tilde{\mathrm{FP}}) = \frac{1}{\pe(K \setminus v_\pi, \tilde{\mathrm{FP}})}
  \]
  with $\tilde{\pi}$ being the remaining part of permutation $\pi$
  after removing $v_\pi$.  Note in particular that $\tilde{\mathrm{FP}}$ is
  valid because
  \[
    X_i \setminus v_\pi \subset X_{i+1} \setminus v_\pi \dots \subset X_\ell \setminus v_\pi
  \]
  and hence the induction hypothesis applies. We conclude:
  \begin{align*}
\mathrm{Pr}(\pi \; | \; K, \mathrm{FP}) &= \mathrm{Pr}(v_{\pi} \; | \; K, \mathrm{FP}) \cdot
                               \mathrm{Pr}(\tilde{\pi} \; | \; K \setminus v_\pi,
                               \tilde{\mathrm{FP}}) \\
              &= \frac{\pe(K \setminus v_\pi, \tilde{\mathrm{FP}})}{\pe(K, \mathrm{FP})
                \cdot \pe(K \setminus v_\pi, \tilde{\mathrm{FP}})} \\
              &= \frac{1}{\pe(K, \mathrm{FP})}.\qedhere
  \end{align*}
\end{proof}

\begin{algorithm}
  \caption{The algorithm uniformly samples an AMO from a UCCG $G$.
   A modified version of Clique-Picking (\texttt{precount}) has to be
   executed on $G$ in advance.}
  \label{alg:sampling}
  \SetKwInOut{Input}{input}\SetKwInOut{Output}{output}
  \DontPrintSemicolon
  \Input{A UCCG $G$.}
  \Output{AMO of $G$.}
  \SetKwFunction{FSample}{sample}
  \SetKwFunction{FClique}{drawclique}
  \SetKwFunction{FPerm}{drawperm}
  \SetKwFunction{FCount}{precount}
  \SetKwFunction{FConcat}{concat}
  \SetKwProg{Fn}{function}{}{end}
  \Fn{\FSample{$G$, $\mathrm{memo}$}}{
    $(K, \mathrm{FP}, \chordalcomps_G(K)) \leftarrow \FClique(\mathrm{memo}[G])$ \;
    $\tau \leftarrow \FPerm(K, \mathrm{FP}, \mathrm{memo})$ \;
    \ForEach{$H \in \chordalcomps_G(K)$}{
      $\tau \leftarrow \FConcat(\tau, \FSample(H, \mathrm{memo}))$ \;
    }
    \KwRet $\tau$ \;
  }
  \;
  $\mathrm{memo} \leftarrow \FCount(G)$ \;
  $\tau \leftarrow \FSample(G, \mathrm{memo})$ \;
  \KwRet AMO of $G$ given by $\tau$ \; 
\end{algorithm}

\begin{claim theorem} 
  There is an algorithm that, given a connected chordal graph $G$, uniformly
  samples AMOs of $G$ in time $\mathcal{O}(|V|+|E|)$ after an initial
  $\mathcal{O}(\cliques(G)^2\cdot |V| \cdot |E|)$ setup.  
\end{claim theorem}

\begin{proof}
  We first prove that Algorithm~\ref{alg:sampling} samples uniformly
  and will take care of the run time afterwards.

  Denote with $\mathrm{Pr}(\tau_{\alpha}(G))$ the
  probability that we draw a topological ordering $\tau$ of the vertices in
  $G$ that represents $\alpha$. We show the theorem by induction, similarly to the proof of
  Theorem~\ref{theorem:cliquepicking}. For the base case of a single clique $K$,
  the algorithm returns a uniformly sampled AMO, as $\mathrm{FP}$ is empty and \texttt{drawperm}
  returns a uniformly sampled permutation. It follows that for each
  AMO $\alpha$:
  \[
    \mathrm{Pr}(\tau_{\alpha}(G)) = 1 / |K|! = 1 / \hamo(G).
  \]
  For an arbitrary graph $G$, let $K_{\alpha}$ be
  the unique clique at which $\alpha$ is considered (such a clique
  exists by the proof of Proposition~\ref{proposition:fpFormula}) with
  forbidden prefixes~$\mathrm{FP}$,
  and $\pi_{\alpha}$ the
  corresponding permutation of $K_\alpha$.
  The term $\mathrm{Pr}(K_{\alpha} \; | \; \mathrm{FP})$  describes the probability that $K_{\alpha}$ is drawn: 
  \begin{align*}
    &\relphantom{=}{} \mathrm{Pr}(\tau_{\alpha}(G)) \\
    &= \mathrm{Pr}(K_{\alpha} \; | \; \mathrm{FP}) \mathrm{Pr}(\pi_{\alpha} \; | \; K_{\alpha}, \mathrm{FP}) \prod_{H \in
      \chordalcomps_G(K_\alpha)} \mathrm{Pr}(\tau_{\alpha[H]}(H)) \\
    &= \frac{\pe(K_{\alpha}, \mathrm{FP}) \prod_{H \in \chordalcomps_G(K_{\alpha})} \hamo(H)
      }{\hamo(G) \cdot \pe(K_{\alpha},
      \mathrm{FP}) \prod_{H \in
      \chordalcomps_G(K_\alpha)}\hamo(H)} \\
    &= \frac{1}{\hamo(G)}.
  \end{align*}
For the second step, we make use of the following facts. By induction
hypothesis we have
\[
  \prod_{H \in \chordalcomps_G(K_\alpha)} \mathrm{Pr}(\tau_{\alpha[H]}(H)) = \frac{1}{\prod_{H \in
      \chordalcomps_G(K_\alpha)} \hamo(H)};
\]
and
\[
  \mathrm{Pr}(K_{\alpha} \; | \; \mathrm{FP}) = \frac{\pe(K_\alpha, \mathrm{FP}) \cdot \prod_{H \in \chordalcomps_G(K_{\alpha})} \hamo(H)}{\hamo(G)}
\]
as $K_{\alpha}$ is picked with probability proportional to the number of
AMOs that are considered at clique $K_{\alpha}$ (function \texttt{drawclique} in
Algorithm~\ref{alg:sampling}); and by Lemma~\ref{lemma:drawpermcorrect}:
\[
  \mathrm{Pr}(\pi_{\alpha} \; | \; K_{\alpha}, \mathrm{FP}) = \frac{1}{\pe(K_\alpha, \mathrm{FP})}.
\] 

We discuss the run time of Algorithm~\ref{alg:sampling}. Assuming
a sampling generator has been initialized during \texttt{precount},
the function \texttt{drawclique}
can be performed in $\mathcal{O}(1)$ with the Alias
Method~\cite{Vose1991}. Assuming, furthermore, pointers to the
subproblems $H \in \mathcal{C}_G(K)$, there is no additional effort in
obtaining those. Since only $\mathcal{O}(|V|)$ subproblems are considered,
these steps take $\mathcal{O}(|V|)$ time in total.

Hence, \texttt{drawperm} is the most expensive step and it remains to
bound its cost. For each place in the permutation~$\pi$, which is
build, the algorithm goes once through all vertices $v$ and assigns
the values $\mathrm{wt}$. We assume that the values for $\phi$ are
precomputed. This is feasible as $\phi$ only depends on two parameters: the size
of the current $K$ and the index $i$ of the remaining forbidden
prefixes $X_i, \dots, X_\ell$. Assuming such precomputations,
\texttt{drawperm} can be implemented in time $\mathcal{O}(|K|^2)$ for a
clique $K$. In total, the
run time is bounded by $\mathcal{O}(|V| + |E|)$ (each vertex is placed once in a
permutation~$\pi$ and the cost are bounded by $|K|$, which is linear
in the number of neighbors of this vertex).

For the precomputation of $\phi$, it is possible to calculate all $\phi$-values in
time $\mathcal{O}(|V| \cdot |E|)$ by dynamic programming (similarly as
in the proof of Theorem~\ref{thm:runtime}). To see this, note that $\phi$ can also be
expressed as:
\[
  \phi(S,R) = |S|!-\sum_{i=1}^{\ell}|X_i|! \cdot\phi(S \setminus
  X_i,\{X_{i+1} \setminus X_i,\dots,X_\ell \setminus X_i\}).
\]
Hence, we have, for each subproblem and
clique, an effort of $\mathcal{O}(|V| \cdot |E|)$. Since
$\phi$ is the costliest precomputation, we conclude that
preprocessing takes in total:
\[
  \mathcal{O}(|\cliques(G)|^2 \cdot |V| \cdot |E|).\qedhere
\]
\end{proof}

We note that, while preprocessing requires an additional factor $|V|$ compared to
Clique-Picking in theory, we observed that the sets
$\mathrm{FP}$ are usually so small that its influence is
negligible for most practical cases.

\begin{claim theorem}
  The problems $\hamo$ and uniform sampling from a Markov
  equivalence class are in $\P$.
\end{claim theorem}

\begin{proof}
  We begin by proving the statement for $\hamo$. It is
  well-known that this problem
  reduces to counting AMOs of UCCGs~\cite{He2015}.
  This is the problem the Clique-Picking algorithm solves. By
  Theorem~\ref{thm:runtime}, it performs polynomially many arithmetic operations. Since $\hamo(G)$ is bounded above
  by $n!$ (let $n$ be $|V|$), all operations run in polynomial time
  because the involved
  numbers can be represented by polynomially many bits:
  \[
    n! \leq n^n = (2^{\log n})^n = 2^{n \cdot \log n}.
  \]
  
  We analyze the complexity of the uniform sampling problem.
  By Theorem~\ref{theorem:samplingtime}, sampling is possible in linear
  time assuming a modified version of Clique-Picking was performed as
  preprocessing step. This preprocessing can be performed in polynomial
  time. Moreover, the bit complexity of sampling is still polynomial
  (by a similar argument as above).
\end{proof}

\emph{Partially directed acyclic graphs} (PDAGs) are, as the name suggests,
partially directed graphs without directed cycles. \emph{Maximally oriented PDAGs}
(MPDAGs) are PDAGs such that \emph{none} of the four Meek rules~\cite{Meek1995}
can be applied. Both graph types are usually used in graphical
modeling to represent background knowledge.

 \begin{claim theorem}
  The problem of counting the number of AMOs is
  \sharpP-complete for PDAGs and MPDAGs.
\end{claim theorem}
 \begin{proof}
   Note that, clearly, both problems are in \sharpP.
   
   We reduce the \sharpP-hard problem of counting the number of
   topological orderings of a DAG~\cite{brightwell1991counting}
   to counting the number of AMOs of a PDAG.  Since we can
   exhaustively apply all Meek rules in polynomial
   time~\cite{Meek1995} without changing the number of AMOs, this
   implies that the problem of counting AMOs is \sharpP-complete for PDAGs and MPDAGs.
   
   Given a DAG $G = (V,E)$, we construct the PDAG $G'$ as follows: $G'$ has the
   same set of vertices $V$ as $G$ and we add all edges from $G$ to
   $G'$. We insert an undirected edge for all pairs of remaining nonadjacent
   vertices in $G'$.

   Each AMO of $G'$ can be represented by exactly one linear ordering
   of $V$ (because $G'$ is complete) and each topological ordering of
   $G$ is a linear ordering as well.
   We prove in two directions that a linear ordering of $V$ is an AMO of $G'$ if, and only if, it
   is a topological ordering of $G$.
   
   \begin{enumerate}
   \item[$\Rightarrow$)] If a linear ordering $\tau$ represents an AMO
     of $G'$, the edges in $G$ are correctly
     reproduced. Hence, it is a topological ordering of $G$.
   \item[$\Leftarrow$)] If a linear ordering $\tau$ is a topological ordering of $G$,
     the orientation of $G'$ according to it is, by definition, acyclic
     and reproduces the directed edges in $G'$. As $G'$ is complete,
     there can be no v-structures. Hence, $\tau$ represents an AMO of $G'$.\qedhere
   \end{enumerate}
 \end{proof}

  Notably, the reason Clique-Picking cannot be used to solve these
  counting problems can be directly
  connected to the hardness proof. Intuitively, these problems can be
  reduced to the setting that, when counting AMOs in UCCGs,
  some edge orientations in the chordal component are
  predetermined by \emph{background knowledge}. Hence, in
  the Clique-Picking algorithm, when counting the number of
  permutations for a clique $K$, we have to count only those consistent with the
  background knowledge. But this is equivalent to the
  problem of counting the number of topological orderings.

 \clearpage 
 \part{Experimental Evaluation\\ of the Clique-Picking Algorithm}\label{section:experiments}
 We evaluated the practical performance of the Clique-Picking algorithm in a
 series of experiments. For this, we compared Clique-Picking, denoted
 (cp) in the following, with the state-of-the-art root picking
 algorithm \emph{AnonMAO} (am)~\cite{Ganian2020}, a
 solver known as \emph{TreeMAO} that performs dynamic programming on a tree decomposition
 (tw)~\cite{Talvitie2019}, and a recently proposed algorithm, which
 combines concepts from intervention design with dynamic
 programming, called \emph{LazyIter} (li)~\cite{Teshnizi20}. With these three competitors, we cover all
 currently known ways to approach the problem of counting Markov equivalent DAGs.

 We omitted the classic root-picking algorithm~\cite{He2015} as
 well as  MemoMAO~\cite{Talvitie2019}, as AnonMAO is an improved
 version of these algorithms, which has been shown to clearly
 outperform them in previous experiments~\cite{Ganian2020}.

 All our experiments were performed on a desktop computer equipped
 with 14 GB of RAM and an Intel Core i7 X980 with 6 cores of 3.33 Ghz
 each. The system runs Ubuntu 20.04 LTS, 64bit. The Clique-Picking
 algorithm was compiled with \texttt{gcc 9.3.0} using \texttt{g++
   -std=c++11 -O3 -march=native -lgmp -lgmpxx}. The dynamic program
 and the solver using tree decomposition were compiled with the same
 version of \texttt{gcc} and the makefiles provided by the
 authors. For the Python solver LazyIter we used
 \texttt{PyPy 7.3.1} (compiled with \texttt{gcc 9.3.0} for
 \texttt{Python~3.6.9}) in order to obtain a compiled and comparable version.

 \section{Generation of Random Chordal Graphs}
 As input for the solvers, we use various classes of random chordal
 graphs, which we explain in detail in the following. The results of
 our experiments can be found in Fig.~\ref{figure:experiments} and
 Fig.~\ref{figure:experiments2}. Beside the time used by the 
 solvers, we also present the number of maximal cliques as well as the
 density $|E|/\binom{|V|}{2}$ that the graphs of the various random graph
 models have. In Fig.~\ref{figure:cfp}, the reader finds a cumulative plot over all experiments performed.
 
 \subsection{Random Subtree Intersections}
 The first family of random chordal graphs that we study is due to~\citet{SekerHET17}. It is
 based on the characterization of chordal graphs as intersection graph
 of subtrees (take a tree, some subtrees of this tree, and build the
 intersection graph of these subtrees~--~the result is a chordal
 graph). The family is generated by an algorithm that obtains two
 parameters as input: $n$ the number of vertices and $k$, a density
 parameter. In the first phase, a random tree on $n$ vertices is
 generated with the algorithm of~\citet{RodionovC04}. Afterward, $n$ subtrees are sampled randomly with
 the following procedure: Initialize the subtree to some random vertex
 $v$, then grow it by adding a random neighbor of the tree until the
 subtree reaches a size $s$ randomly drawn from
 $\{1,\dots,2k-1\}$. Finally, we take the intersection graph of these
 $n$ subtrees, which is a chordal graph on $n$ vertices.

 We performed three experiments on this graph class. In each
 experiment we choose $n=2^i$ for $i\in\{4,\dots,12\}$ and $k$ as
 either $k=\log n$ (Fig.~\ref{figure:experiments}.A), $k=2\log n$ (Fig.~\ref{figure:experiments}.B), or
 $k=\sqrt{n}$ (Fig.~\ref{figure:experiments}.C). In each case, we generated ten random graphs
 for every value of $n$ and took the mean time used by the
 algorithms. As mentioned above, we also display the number of maximal cliques
 in the graph and the mean density of the graphs, i.\,e.\ $|E| / {|V| \choose 2}$.
 
 Our solver (cp) clearly performs best, solving even the largest
 instances with 4096 vertices in a short amount of time. This also
 holds for denser graphs, e.\,g.\ with $k = \sqrt{n}$. We can observe
 that the run time of (cp) only increases slowly for higher
 densities. The solver (am) can be clearly distinguished as second
 best, however, for larger and even moderately dense graphs the method
 timeouts, i.\,e., graphs with more than 256 vertices are infeasible
 for $k=2\log{n}$ and $k=\sqrt{n}$. The remaining methods (tw) and (li)
 perform comparably poorly and are only able to solve graphs with less
 than 100 vertices. Interestingly, the run time of both algorithms has
 an extremely high variance. For the case of~(tw), this might be
 explained by the fact that the run time depends superexponentially on the
 size of the largest clique (denoted $c$) including a factor $c! \cdot 2^c
 \cdot c^2$.

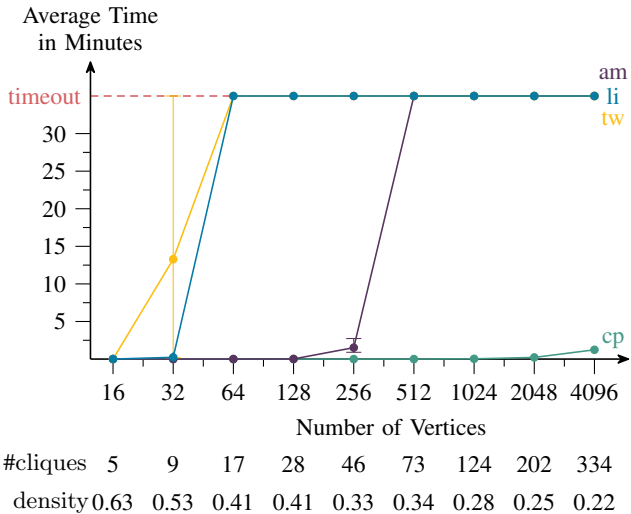
 \subsection{Random Interval Graphs}
 The next class of random chordal graphs that we consider are \emph{random interval
   graphs.} An interval graph is the intersection graph of a set of
 intervals. It is well-known that interval graphs are a subclass of
 chordal graphs and, thus, they are well-suited as input for our
 experiments. The solver (am) was proven to run in polynomial time on
 interval graphs~\cite{Ganian2020}, which makes for an interesting
 comparison.

 We generate random interval
 graphs with the algorithm by Scheinerman~\shortcite{Scheinerman88}, which
 simply draws~$2n$ random variables in $[0,1]$, pairs these as
 intervals, and build the intersection graph of these intervals.

 As in the previous section, we generated for every value $i\in\{4,\dots,12\}$ ten
 random interval graphs on $2^i$ vertices. We ran all algorithms on
 these instances and built the mean of their runtime for each $i$. The
 results can be found in Fig.~\ref{figure:experiments}.D.

 The instances of this test set are the densest ones that we considered
 during our experiments, with about two thirds of the possible edges
 being present in the graph. We observe that (tw) and (li) are only able to solve the smallest
 instances. The (am) algorithm manages to solve
 instances up to 128 vertices, but timeouts for larger ones. While, as
 mentioned above, the
 run time of (am) is bounded by a polynomial for this class of graphs, the
 degree of the polynomial appears to be too large for practical use. In contrast, the (cp) algorithm manages to
 solve \emph{all} instances, though it has to be noted that it takes significant time (around 10
 minutes) for the largest ones with 4096 vertices. An explanation for this is the higher number
 of maximal cliques compared to the first set of experiments.
 Recall that the number of recursively explored subgraphs 
 is bounded by $2\cdot |\cliques(G)| - 1$. Since the graphs are dense, these
 subgraphs are also relatively large.

 \subsection{Random Perfect Elimination Orderings}
 Another characterization of chordal graphs is based on \emph{perfect
   elimination orderings,} which we have discussed in
 Section~\ref{sec:lbfsmaos}. Recall that a perfect elimination ordering
 of a graph $G=(V,E)$ is a permutation $\pi$ of its vertices such that
 for each $v\in V$ the set $\{\,w\mid \pi(v)<\pi(w)\,\}$ is a
 clique. It is well-known that a graph is chordal if, and only if, it
 has a perfect elimination ordering~\cite{Blair1993}. This leads to a
 simple randomized algorithm to generate chordal graphs: Start with a graph of $n$ isolated
 vertices $v_1,\dots,v_n$; iterate over the vertices in this order and
 for every $v_i$ initialize a set $S=\{\,v_j\mid v_j\in N(v_i)\wedge
 j>i\,\}$. Randomly choose a size $s$ and, while $|S|<s$, add some
 further vertices~$v_j$ with $j>i$ to $S$; finally make $S$ a clique
 and adjacent to $v_i$. This algorithm (as well as other refined
 versions of it) tends to produce quite dense
 graphs~\cite{SekerHET17}. Therefore, we use \emph{constant size} bounds
 for the set $S$ rather than a function that grows with~$n$. More
 precisely, we run the algorithm with a parameter $k$ and let it pick
 for every vertex a size bound $s\in\{k/2,\dots,2k\}$ at random. Note that
 this does \emph{not} imply that the degree is bounded by $2k$.

 We performed the same experiments as in the previous
 sections on this graph class for $k=2$ and $k=4$. The results can
 be found in Fig.~\ref{figure:experiments}.E and
 Fig.~\ref{figure:experiments}.F, respectively.

 The instances generated with this method are quite interesting,
 because they have significantly different properties than, for example,
 the ones generated with the subtree intersection method. In
 particular, there is a large number of maximal cliques. Hence, these
 test sets contain some of the hardest instances for (cp). Still, the
 algorithm outperforms the other methods by a large margin.
 Only for the largest graphs with 4096 vertices and density parameter $k=4$, there are some
 instances that (cp) could not solve within 30 minutes.

 Notably, the (am) algorithm, which again takes the second place,
 performs worse on these instances compared to the ones generated with
 the subtree intersection method, too. For instance in
 Fig.~\ref{figure:experiments}.E the graphs are sparser than in
 Fig.~\ref{figure:experiments}.A, but (am) still takes more time; the
 same holds for Fig.~\ref{figure:experiments}.F and
 Fig.~\ref{figure:experiments}.B. This makes sense, as the number of
 subproblems of (am) was shown to depend on the size of the clique
 tree as well~\cite{Ganian2020}. The other two
 methods are less affected by the increase of maximal
 cliques and perform similar to the other experiments.
 
 \subsection{Random Tree Thickening}\label{section:thickening}
 In this section, we mimic the experiments performed by He, Jia, and
 Yu~\shortcite{He2015}, Talvitie and Koivisto~\shortcite{Talvitie2019}, and Ganian, Hamm, and
 Talvitie~\shortcite{Ganian2020}. \citet{He2015} proposed the following
 way of generating random chordal graphs with $n$ vertices and $kn$
 edges, which we call \emph{random tree thickening}. We start with a
 random tree on $n$ vertices, which is generated by the
 Rodionov-Choo-Algorithm~\cite{RodionovC04}. Then, as long as the
 graph in construction has less than $kn$ edges, we repeatedly pick two
 non-adjacent vertices at random and connect them by an edge if the
 resulting graph is still chordal.

 The advantage of this way of generating random chordal graphs is that
 we have fine control over the density of the generated
 graphs. However, the sketched algorithm is by far the slowest random
 chordal graph generator and becomes infeasible for larger values of
 $n$. We have therefore restricted the experiments in this sections to
 graphs with $n=2^i$  vertices for $i\in\{4,\dots,10\}$. As density
 parameter $k$ we used $k=3$ (Fig.~\ref{figure:experiments2}.A), $k=\log n$
 (Fig.~\ref{figure:experiments2}.B), and $k=\sqrt n$
 (Fig.~\ref{figure:experiments2}.C).

 As the instances generated with this method are comparatively small
 and sparse, they are easy instances for (cp), which solves them all
 in seconds. For the other methods, the same insights as before hold.

 \section{Summary of the Experiments}
 The results of our experiments reveal that the Clique-Picking
 approach is superior to all state-of-the-art strategies for counting
 MAOs of chordal graphs by orders of magnitudes. In particular, the Clique-Picking algorithm
 even outperforms other solvers on instances on which they should
 naturally perform quite strong: It is not only faster on dense, but
 also sparse instances, and while (am) runs in polynomial time on interval graphs,
 it is still much slower than Clique-Picking on these graphs. 

  \begin{figure}
   \input{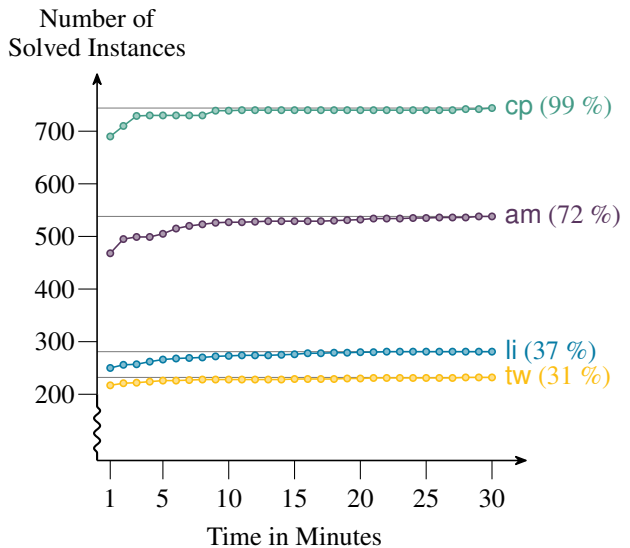}
   \caption{A cumulative distribution function plot that shows the
     performance of the four solvers
     \textcolor{ba.pine}{Clique-Picking (cp)},
     \textcolor{ba.violet}{AnonMAO (am)},
     \textcolor{ba.yellow}{TreeMAO (tw)}, and
     \textcolor{ba.blue}{LazyIter (li)} on all 750 random chordal
     graphs that where used for experiments in this section.  The
     $x$-axis shows the time in minutes, and the $y$-axis the number
     of instances the solver solved (having $x$ minutes per
     instance).}
   \label{figure:cfp}
 \end{figure}
 
 We summarize our experiments in
 Fig.~\ref{figure:cfp}, which contains a \emph{cumulative
   distribution function plot} over all experiments that we
 performed. Overall. we produced 750 random chordal graphs
 for the experiments in this section.

 As seen in the previous sections, the solvers (tw) and (li) take the
 last places in our competition. It has to be noted that (li) even produced
 incorrect answers for a few instances (we double-checked these by
 brute-force enumeration). The (am) solver takes second
 place solving $72\%$ of the instances, struggling in particular with
 larger and denser graphs.

 From the 750 graphs, there were only six
 which the Clique-Picking algorithm could not solve in 30 minutes
 (these were dense graphs generated by the random elimination
 ordering algorithm with $k=4$, which have a relatively large number of maximal
 cliques). Moreover, (cp) was also extremely fast in solving most of
 the instances, almost 700 of them were solved in one minute or less (most of them in a few
 seconds), while the second-best method (am) solved significantly less instances even in the
 whole 30 minutes. This shows that the Clique-Picking algorithm is not only of
 theoretical value, but currently the by far most practical algorithm for
 counting Markov equivalent DAGs.

 \begin{figure*}
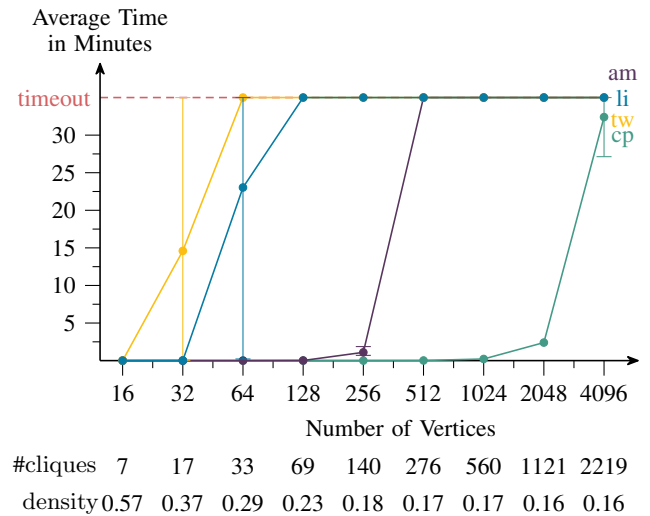

   \input{img/subtree-logn.tikz}\quad\input{img/subtree-2logn.tikz}\\[0.25cm]
   \input{img/subtree-sqrtn.tikz}\quad\input{img/interval.tikz}\\[0.25cm]
   \input{img/peo-2.tikz}\quad\input{img/peo-4.tikz}
   \caption{Comparison of the solvers
     \textcolor{ba.pine}{Clique-Picking (cp)},
     \textcolor{ba.violet}{AnonMAO (am)},
     \textcolor{ba.yellow}{TreeMAO (tw)}, and
     \textcolor{ba.blue}{LazyIter (li)}. The plots show the
     mean time (dot) as well as the minimum and maximum time (thin
     lines) each solver needed over \emph{ten} random graphs. A solver
     that requires more than 30 minutes or runs out of memory obtains
     a \textcolor{ba.red}{timeout}, which contributes 35 minutes to
     the mean time. At the bottom of each plot we present the average number of
     maximal cliques as well as the average density.}
   \label{figure:experiments}
 \end{figure*}

 \begin{figure*}
   \centering
   \begin{tikzpicture}
  \draw[tick] (0.8,0) -- (0.8,-0.2) node[below] {16};
  \draw[tick] (1.6,0) -- (1.6,-0.2) node[below] {32};
  \draw[tick] (1.2000000000000002,0) -- (1.2000000000000002,-0.1);
  \draw[tick] (2.4000000000000004,0) -- (2.4000000000000004,-0.2) node[below] {64};
  \draw[tick] (2,0) -- (2,-0.1);
  \draw[tick] (3.2,0) -- (3.2,-0.2) node[below] {128};
  \draw[tick] (2.8000000000000003,0) -- (2.8000000000000003,-0.1);
  \draw[tick] (4,0) -- (4,-0.2) node[below] {256};
  \draw[tick] (3.6,0) -- (3.6,-0.1);
  \draw[tick] (4.800000000000001,0) -- (4.800000000000001,-0.2) node[below] {512};
  \draw[tick] (4.4,0) -- (4.4,-0.1);
  \draw[tick] (5.6000000000000005,0) -- (5.6000000000000005,-0.2) node[below] {1024};
  \draw[tick] (5.2,0) -- (5.2,-0.1);
  \draw[tick] (0.5,0.5) -- (0.3,0.5) node[left] {5};
  \draw[tick] (0.5,0.25) -- (0.4,0.25);
  \draw[tick] (0.5,1) -- (0.3,1) node[left] {10};
  \draw[tick] (0.5,0.75) -- (0.4,0.75);
  \draw[tick] (0.5,1.5) -- (0.3,1.5) node[left] {15};
  \draw[tick] (0.5,1.25) -- (0.4,1.25);
  \draw[tick] (0.5,2) -- (0.3,2) node[left] {20};
  \draw[tick] (0.5,1.75) -- (0.4,1.75);
  \draw[tick] (0.5,2.5) -- (0.3,2.5) node[left] {25};
  \draw[tick] (0.5,2.25) -- (0.4,2.25);
  \draw[tick] (0.5,3) -- (0.3,3) node[left] {30};
  \draw[tick] (0.5,2.75) -- (0.4,2.75);
  \node[tick] at (-0.1, -1.4) {\#$\mathrm{cliques}$};
  \node[tick] at (-0.1, -1.9) {\phantom{\#}$\mathrm{density}$};
  \node[tick] at (0.8,-1.4) {10};
  \node[tick] at (0.8,-1.9) {0.40};
  \node[tick] at (1.6,-1.4) {24};
  \node[tick] at (1.6,-1.9) {0.19};
  \node[tick] at (2.4000000000000004,-1.4) {50};
  \node[tick] at (2.4000000000000004,-1.9) {0.10};
  \node[tick] at (3.2,-1.4) {106};
  \node[tick] at (3.2,-1.9) {0.05};
  \node[tick] at (4,-1.4) {212};
  \node[tick] at (4,-1.9) {0.02};
  \node[tick] at (4.800000000000001,-1.4) {428};
  \node[tick] at (4.800000000000001,-1.9) {0.01};
  \node[tick] at (5.6000000000000005,-1.4) {865};
  \node[tick] at (5.6000000000000005,-1.9) {0.01};
  \node[tick, text width=7cm, align=center] at (3.5, -0.9) {Number of Vertices};
  \node[tick, text width=2cm, align=center] at (0.5, 4.4) {Average Time\\ in Minutes};
  \node[baseline, fill=ba.gray!50, inner sep=0.5ex] at (0, 5.2) {A};
  \node[baseline, anchor=west, inner sep=0.5ex] at (0.33, 5.175) {Tree Thickening, $k=3$};
  \draw[timeout] (0.5, 3.5) node[left] {timeout} -- (5.6000000000000005, 3.5);
  \draw[axis] (0.5,0) to (6.1000000000000005, 0);
  \draw[axis] (0.5,0) to (0.5, 4);
  \draw[semithick, color=ba.pine] (0.8,0) -- (1.6,0);
  \draw[semithick, color=ba.pine] (1.6,0) -- (2.4000000000000004,0);
  \draw[semithick, color=ba.pine] (2.4000000000000004,0) -- (3.2,0);
  \draw[semithick, color=ba.pine] (3.2,0) -- (4,0);
  \draw[semithick, color=ba.pine] (4,0) -- (4.800000000000001,0);
  \draw[semithick, color=ba.pine] (4.800000000000001,0) -- (5.6000000000000005,0.0016666666666666668);
  \node[mean_dot, color=ba.pine] at (0.8,0) {};
  \draw[thin, color=ba.pine] (0.8,0) -- (0.8,0);
  \node[mean_dot, color=ba.pine] at (1.6,0) {};
  \draw[thin, color=ba.pine] (1.6,0) -- (1.6,0);
  \node[mean_dot, color=ba.pine] at (2.4000000000000004,0) {};
  \draw[thin, color=ba.pine] (2.4000000000000004,0) -- (2.4000000000000004,0);
  \node[mean_dot, color=ba.pine] at (3.2,0) {};
  \draw[thin, color=ba.pine] (3.2,0) -- (3.2,0);
  \node[mean_dot, color=ba.pine] at (4,0) {};
  \draw[thin, color=ba.pine] (4,0) -- (4,0.0016666666666666668);
  \node[mean_dot, color=ba.pine] at (4.800000000000001,0) {};
  \draw[thin, color=ba.pine] (4.800000000000001,0) -- (4.800000000000001,0.0016666666666666668);
  \node[mean_dot, color=ba.pine] at (5.6000000000000005,0.0016666666666666668) {};
  \draw[thin, color=ba.pine] (5.6000000000000005,0.0016666666666666668) -- (5.6000000000000005,0.0033333333333333335);
  \node[tick, color=ba.pine] at (5.8500000000000005,-0.19833333333333333) {cp};
  \draw[semithick, color=ba.violet] (0.8,0) -- (1.6,0);
  \draw[semithick, color=ba.violet] (1.6,0) -- (2.4000000000000004,0);
  \draw[semithick, color=ba.violet] (2.4000000000000004,0) -- (3.2,0);
  \draw[semithick, color=ba.violet] (3.2,0) -- (4,0);
  \draw[semithick, color=ba.violet] (4,0) -- (4.800000000000001,0.0016666666666666668);
  \draw[semithick, color=ba.violet] (4.800000000000001,0.0016666666666666668) -- (5.6000000000000005,0.03833333333333334);
  \node[mean_dot, color=ba.violet] at (0.8,0) {};
  \draw[thin, color=ba.violet] (0.8,0) -- (0.8,0.0016666666666666668);
  \node[mean_dot, color=ba.violet] at (1.6,0) {};
  \draw[thin, color=ba.violet] (1.6,0) -- (1.6,0);
  \node[mean_dot, color=ba.violet] at (2.4000000000000004,0) {};
  \draw[thin, color=ba.violet] (2.4000000000000004,0) -- (2.4000000000000004,0);
  \node[mean_dot, color=ba.violet] at (3.2,0) {};
  \draw[thin, color=ba.violet] (3.2,0) -- (3.2,0);
  \node[mean_dot, color=ba.violet] at (4,0) {};
  \draw[thin, color=ba.violet] (4,0) -- (4,0.0016666666666666668);
  \node[mean_dot, color=ba.violet] at (4.800000000000001,0.0016666666666666668) {};
  \draw[thin, color=ba.violet] (4.800000000000001,0.0016666666666666668) -- (4.800000000000001,0.0033333333333333335);
  \node[mean_dot, color=ba.violet] at (5.6000000000000005,0.03833333333333334) {};
  \draw[thin, color=ba.violet] (5.6000000000000005,0.03833333333333334) -- (5.6000000000000005,0.04166666666666667);
  \node[tick, color=ba.violet] at (5.8500000000000005,0.43833333333333335) {am};
  \draw[semithick, color=ba.yellow] (0.8,0) -- (1.6,0);
  \draw[semithick, color=ba.yellow] (1.6,0) -- (2.4000000000000004,0);
  \draw[semithick, color=ba.yellow] (2.4000000000000004,0) -- (3.2,0);
  \draw[semithick, color=ba.yellow] (3.2,0) -- (4,0);
  \draw[semithick, color=ba.yellow] (4,0) -- (4.800000000000001,0);
  \draw[semithick, color=ba.yellow] (4.800000000000001,0) -- (5.6000000000000005,0);
  \node[mean_dot, color=ba.yellow] at (0.8,0) {};
  \draw[thin, color=ba.yellow] (0.8,0) -- (0.8,0);
  \node[mean_dot, color=ba.yellow] at (1.6,0) {};
  \draw[thin, color=ba.yellow] (1.6,0) -- (1.6,0);
  \node[mean_dot, color=ba.yellow] at (2.4000000000000004,0) {};
  \draw[thin, color=ba.yellow] (2.4000000000000004,0) -- (2.4000000000000004,0);
  \node[mean_dot, color=ba.yellow] at (3.2,0) {};
  \draw[thin, color=ba.yellow] (3.2,0) -- (3.2,0);
  \node[mean_dot, color=ba.yellow] at (4,0) {};
  \draw[thin, color=ba.yellow] (4,0) -- (4,0.0016666666666666668);
  \node[mean_dot, color=ba.yellow] at (4.800000000000001,0) {};
  \draw[thin, color=ba.yellow] (4.800000000000001,0) -- (4.800000000000001,0.0016666666666666668);
  \node[mean_dot, color=ba.yellow] at (5.6000000000000005,0) {};
  \draw[thin, color=ba.yellow] (5.6000000000000005,0) -- (5.6000000000000005,0.0016666666666666668);
  \node[tick, color=ba.yellow] at (5.8500000000000005,0.2) {tw};
  \draw[semithick, color=ba.blue] (0.8,0) -- (1.6,0);
  \draw[semithick, color=ba.blue] (1.6,0) -- (2.4000000000000004,0.0016666666666666668);
  \draw[semithick, color=ba.blue] (2.4000000000000004,0.0016666666666666668) -- (3.2,0.0016666666666666668);
  \draw[semithick, color=ba.blue] (3.2,0.0016666666666666668) -- (4,0.006666666666666667);
  \draw[semithick, color=ba.blue] (4,0.006666666666666667) -- (4.800000000000001,3.1550000000000002);
  \draw[semithick, color=ba.blue] (4.800000000000001,3.1550000000000002) -- (5.6000000000000005,3.5);
  \node[mean_dot, color=ba.blue] at (0.8,0) {};
  \draw[thin, color=ba.blue] (0.8,0) -- (0.8,0.0016666666666666668);
  \node[mean_dot, color=ba.blue] at (1.6,0) {};
  \draw[thin, color=ba.blue] (1.6,0) -- (1.6,0.0016666666666666668);
  \node[mean_dot, color=ba.blue] at (2.4000000000000004,0.0016666666666666668) {};
  \draw[thin, color=ba.blue] (2.4000000000000004,0) -- (2.4000000000000004,0.0033333333333333335);
  \node[mean_dot, color=ba.blue] at (3.2,0.0016666666666666668) {};
  \draw[thin, color=ba.blue] (3.2,0.0016666666666666668) -- (3.2,0.005000000000000001);
  \node[mean_dot, color=ba.blue] at (4,0.006666666666666667) {};
  \draw[thin, color=ba.blue] (4,0.005000000000000001) -- (4,0.010000000000000002);
  \node[mean_dot, color=ba.blue] at (4.800000000000001,3.1550000000000002) {};
  \draw[thin, color=ba.blue] (4.800000000000001,0.05666666666666667) -- (4.800000000000001,3.5);
  \draw[thin, color=ba.blue] (4.700000000000001,0.05666666666666667) -- (4.9,0.05666666666666667);
  \draw[thin, color=ba.blue] (4.700000000000001,3.5) -- (4.9,3.5);
  \node[mean_dot, color=ba.blue] at (5.6000000000000005,3.5) {};
  \draw[thin, color=ba.blue] (5.6000000000000005,3.5) -- (5.6000000000000005,3.5);
  \node[tick, color=ba.blue] at (5.8500000000000005,3.5) {li};
\end{tikzpicture}\quad\begin{tikzpicture}
  \draw[tick] (0.8,0) -- (0.8,-0.2) node[below] {16};
  \draw[tick] (1.6,0) -- (1.6,-0.2) node[below] {32};
  \draw[tick] (1.2000000000000002,0) -- (1.2000000000000002,-0.1);
  \draw[tick] (2.4000000000000004,0) -- (2.4000000000000004,-0.2) node[below] {64};
  \draw[tick] (2,0) -- (2,-0.1);
  \draw[tick] (3.2,0) -- (3.2,-0.2) node[below] {128};
  \draw[tick] (2.8000000000000003,0) -- (2.8000000000000003,-0.1);
  \draw[tick] (4,0) -- (4,-0.2) node[below] {256};
  \draw[tick] (3.6,0) -- (3.6,-0.1);
  \draw[tick] (4.800000000000001,0) -- (4.800000000000001,-0.2) node[below] {512};
  \draw[tick] (4.4,0) -- (4.4,-0.1);
  \draw[tick] (5.6000000000000005,0) -- (5.6000000000000005,-0.2) node[below] {1024};
  \draw[tick] (5.2,0) -- (5.2,-0.1);
  \draw[tick] (0.5,0.5) -- (0.3,0.5) node[left] {5};
  \draw[tick] (0.5,0.25) -- (0.4,0.25);
  \draw[tick] (0.5,1) -- (0.3,1) node[left] {10};
  \draw[tick] (0.5,0.75) -- (0.4,0.75);
  \draw[tick] (0.5,1.5) -- (0.3,1.5) node[left] {15};
  \draw[tick] (0.5,1.25) -- (0.4,1.25);
  \draw[tick] (0.5,2) -- (0.3,2) node[left] {20};
  \draw[tick] (0.5,1.75) -- (0.4,1.75);
  \draw[tick] (0.5,2.5) -- (0.3,2.5) node[left] {25};
  \draw[tick] (0.5,2.25) -- (0.4,2.25);
  \draw[tick] (0.5,3) -- (0.3,3) node[left] {30};
  \draw[tick] (0.5,2.75) -- (0.4,2.75);
  \node[tick] at (-0.1, -1.4) {\#$\mathrm{cliques}$};
  \node[tick] at (-0.1, -1.9) {\phantom{\#}$\mathrm{density}$};
  \node[tick] at (0.8,-1.4) {9};
  \node[tick] at (0.8,-1.9) {0.53};
  \node[tick] at (1.6,-1.4) {22};
  \node[tick] at (1.6,-1.9) {0.32};
  \node[tick] at (2.4000000000000004,-1.4) {48};
  \node[tick] at (2.4000000000000004,-1.9) {0.19};
  \node[tick] at (3.2,-1.4) {97};
  \node[tick] at (3.2,-1.9) {0.11};
  \node[tick] at (4,-1.4) {200};
  \node[tick] at (4,-1.9) {0.06};
  \node[tick] at (4.800000000000001,-1.4) {401};
  \node[tick] at (4.800000000000001,-1.9) {0.04};
  \node[tick] at (5.6000000000000005,-1.4) {805};
  \node[tick] at (5.6000000000000005,-1.9) {0.02};
  \node[tick, text width=7cm, align=center] at (3.5, -0.9) {Number of Vertices};
  \node[tick, text width=2cm, align=center] at (0.5, 4.4) {Average Time\\ in Minutes};
  \node[baseline, fill=ba.gray!50, inner sep=0.5ex] at (0, 5.2) {B};
  \node[baseline, anchor=west, inner sep=0.5ex] at (0.33, 5.175) {Tree Thickening, $k=\log n$};
  \draw[timeout] (0.5, 3.5) node[left] {timeout} -- (5.6000000000000005, 3.5);
  \draw[axis] (0.5,0) to (6.1000000000000005, 0);
  \draw[axis] (0.5,0) to (0.5, 4);
  \draw[semithick, color=ba.pine] (0.8,0) -- (1.6,0);
  \draw[semithick, color=ba.pine] (1.6,0) -- (2.4000000000000004,0);
  \draw[semithick, color=ba.pine] (2.4000000000000004,0) -- (3.2,0);
  \draw[semithick, color=ba.pine] (3.2,0) -- (4,0);
  \draw[semithick, color=ba.pine] (4,0) -- (4.800000000000001,0);
  \draw[semithick, color=ba.pine] (4.800000000000001,0) -- (5.6000000000000005,0.0016666666666666668);
  \node[mean_dot, color=ba.pine] at (0.8,0) {};
  \draw[thin, color=ba.pine] (0.8,0) -- (0.8,0);
  \node[mean_dot, color=ba.pine] at (1.6,0) {};
  \draw[thin, color=ba.pine] (1.6,0) -- (1.6,0);
  \node[mean_dot, color=ba.pine] at (2.4000000000000004,0) {};
  \draw[thin, color=ba.pine] (2.4000000000000004,0) -- (2.4000000000000004,0);
  \node[mean_dot, color=ba.pine] at (3.2,0) {};
  \draw[thin, color=ba.pine] (3.2,0) -- (3.2,0);
  \node[mean_dot, color=ba.pine] at (4,0) {};
  \draw[thin, color=ba.pine] (4,0) -- (4,0.0016666666666666668);
  \node[mean_dot, color=ba.pine] at (4.800000000000001,0) {};
  \draw[thin, color=ba.pine] (4.800000000000001,0) -- (4.800000000000001,0.0016666666666666668);
  \node[mean_dot, color=ba.pine] at (5.6000000000000005,0.0016666666666666668) {};
  \draw[thin, color=ba.pine] (5.6000000000000005,0.0016666666666666668) -- (5.6000000000000005,0.0033333333333333335);
  \node[tick, color=ba.pine] at (5.8500000000000005,0.15166666666666667) {cp};
  \draw[semithick, color=ba.violet] (0.8,0) -- (1.6,0);
  \draw[semithick, color=ba.violet] (1.6,0) -- (2.4000000000000004,0);
  \draw[semithick, color=ba.violet] (2.4000000000000004,0) -- (3.2,0);
  \draw[semithick, color=ba.violet] (3.2,0) -- (4,0.0016666666666666668);
  \draw[semithick, color=ba.violet] (4,0.0016666666666666668) -- (4.800000000000001,0.016666666666666666);
  \draw[semithick, color=ba.violet] (4.800000000000001,0.016666666666666666) -- (5.6000000000000005,0.18500000000000003);
  \node[mean_dot, color=ba.violet] at (0.8,0) {};
  \draw[thin, color=ba.violet] (0.8,0) -- (0.8,0);
  \node[mean_dot, color=ba.violet] at (1.6,0) {};
  \draw[thin, color=ba.violet] (1.6,0) -- (1.6,0.0016666666666666668);
  \node[mean_dot, color=ba.violet] at (2.4000000000000004,0) {};
  \draw[thin, color=ba.violet] (2.4000000000000004,0) -- (2.4000000000000004,0);
  \node[mean_dot, color=ba.violet] at (3.2,0) {};
  \draw[thin, color=ba.violet] (3.2,0) -- (3.2,0.0016666666666666668);
  \node[mean_dot, color=ba.violet] at (4,0.0016666666666666668) {};
  \draw[thin, color=ba.violet] (4,0) -- (4,0.0033333333333333335);
  \node[mean_dot, color=ba.violet] at (4.800000000000001,0.016666666666666666) {};
  \draw[thin, color=ba.violet] (4.800000000000001,0.016666666666666666) -- (4.800000000000001,0.018333333333333333);
  \node[mean_dot, color=ba.violet] at (5.6000000000000005,0.18500000000000003) {};
  \draw[thin, color=ba.violet] (5.6000000000000005,0.16833333333333333) -- (5.6000000000000005,0.19666666666666666);
  \node[tick, color=ba.violet] at (5.8500000000000005,0.43500000000000005) {am};
  \draw[semithick, color=ba.yellow] (0.8,0) -- (1.6,0);
  \draw[semithick, color=ba.yellow] (1.6,0) -- (2.4000000000000004,0.020000000000000004);
  \draw[semithick, color=ba.yellow] (2.4000000000000004,0.020000000000000004) -- (3.2,2.2183333333333333);
  \draw[semithick, color=ba.yellow] (3.2,2.2183333333333333) -- (4,3.5);
  \draw[semithick, color=ba.yellow] (4,3.5) -- (4.800000000000001,3.5);
  \draw[semithick, color=ba.yellow] (4.800000000000001,3.5) -- (5.6000000000000005,3.5);
  \node[mean_dot, color=ba.yellow] at (0.8,0) {};
  \draw[thin, color=ba.yellow] (0.8,0) -- (0.8,0.0016666666666666668);
  \node[mean_dot, color=ba.yellow] at (1.6,0) {};
  \draw[thin, color=ba.yellow] (1.6,0) -- (1.6,0.0033333333333333335);
  \node[mean_dot, color=ba.yellow] at (2.4000000000000004,0.020000000000000004) {};
  \draw[thin, color=ba.yellow] (2.4000000000000004,0.0016666666666666668) -- (2.4000000000000004,0.08166666666666667);
  \node[mean_dot, color=ba.yellow] at (3.2,2.2183333333333333) {};
  \draw[thin, color=ba.yellow] (3.2,0.14333333333333334) -- (3.2,3.5);
  \draw[thin, color=ba.yellow] (3.1,0.14333333333333334) -- (3.3000000000000003,0.14333333333333334);
  \draw[thin, color=ba.yellow] (3.1,3.5) -- (3.3000000000000003,3.5);
  \node[mean_dot, color=ba.yellow] at (4,3.5) {};
  \draw[thin, color=ba.yellow] (4,3.5) -- (4,3.5);
  \node[mean_dot, color=ba.yellow] at (4.800000000000001,3.5) {};
  \draw[thin, color=ba.yellow] (4.800000000000001,3.5) -- (4.800000000000001,3.5);
  \node[mean_dot, color=ba.yellow] at (5.6000000000000005,3.5) {};
  \draw[thin, color=ba.yellow] (5.6000000000000005,3.5) -- (5.6000000000000005,3.5);
  \node[tick, color=ba.yellow] at (5.8500000000000005,3.35) {tw};
  \draw[semithick, color=ba.blue] (0.8,0) -- (1.6,0.0016666666666666668);
  \draw[semithick, color=ba.blue] (1.6,0.0016666666666666668) -- (2.4000000000000004,0.0033333333333333335);
  \draw[semithick, color=ba.blue] (2.4000000000000004,0.0033333333333333335) -- (3.2,0.021666666666666667);
  \draw[semithick, color=ba.blue] (3.2,0.021666666666666667) -- (4,3.5);
  \draw[semithick, color=ba.blue] (4,3.5) -- (4.800000000000001,3.5);
  \draw[semithick, color=ba.blue] (4.800000000000001,3.5) -- (5.6000000000000005,3.5);
  \node[mean_dot, color=ba.blue] at (0.8,0) {};
  \draw[thin, color=ba.blue] (0.8,0) -- (0.8,0.0016666666666666668);
  \node[mean_dot, color=ba.blue] at (1.6,0.0016666666666666668) {};
  \draw[thin, color=ba.blue] (1.6,0.0016666666666666668) -- (1.6,0.0016666666666666668);
  \node[mean_dot, color=ba.blue] at (2.4000000000000004,0.0033333333333333335) {};
  \draw[thin, color=ba.blue] (2.4000000000000004,0.0033333333333333335) -- (2.4000000000000004,0.006666666666666667);
  \node[mean_dot, color=ba.blue] at (3.2,0.021666666666666667) {};
  \draw[thin, color=ba.blue] (3.2,0.015) -- (3.2,0.04166666666666667);
  \node[mean_dot, color=ba.blue] at (4,3.5) {};
  \draw[thin, color=ba.blue] (4,3.5) -- (4,3.5);
  \node[mean_dot, color=ba.blue] at (4.800000000000001,3.5) {};
  \draw[thin, color=ba.blue] (4.800000000000001,3.5) -- (4.800000000000001,3.5);
  \node[mean_dot, color=ba.blue] at (5.6000000000000005,3.5) {};
  \draw[thin, color=ba.blue] (5.6000000000000005,3.5) -- (5.6000000000000005,3.5);
  \node[tick, color=ba.blue] at (5.8500000000000005,3.65) {li};
\end{tikzpicture}\\[0.25cm]
   \begin{tikzpicture}
  \draw[tick] (0.8,0) -- (0.8,-0.2) node[below] {16};
  \draw[tick] (1.6,0) -- (1.6,-0.2) node[below] {32};
  \draw[tick] (1.2000000000000002,0) -- (1.2000000000000002,-0.1);
  \draw[tick] (2.4000000000000004,0) -- (2.4000000000000004,-0.2) node[below] {64};
  \draw[tick] (2,0) -- (2,-0.1);
  \draw[tick] (3.2,0) -- (3.2,-0.2) node[below] {128};
  \draw[tick] (2.8000000000000003,0) -- (2.8000000000000003,-0.1);
  \draw[tick] (4,0) -- (4,-0.2) node[below] {256};
  \draw[tick] (3.6,0) -- (3.6,-0.1);
  \draw[tick] (4.800000000000001,0) -- (4.800000000000001,-0.2) node[below] {512};
  \draw[tick] (4.4,0) -- (4.4,-0.1);
  \draw[tick] (5.6000000000000005,0) -- (5.6000000000000005,-0.2) node[below] {1024};
  \draw[tick] (5.2,0) -- (5.2,-0.1);
  \draw[tick] (0.5,0.5) -- (0.3,0.5) node[left] {5};
  \draw[tick] (0.5,0.25) -- (0.4,0.25);
  \draw[tick] (0.5,1) -- (0.3,1) node[left] {10};
  \draw[tick] (0.5,0.75) -- (0.4,0.75);
  \draw[tick] (0.5,1.5) -- (0.3,1.5) node[left] {15};
  \draw[tick] (0.5,1.25) -- (0.4,1.25);
  \draw[tick] (0.5,2) -- (0.3,2) node[left] {20};
  \draw[tick] (0.5,1.75) -- (0.4,1.75);
  \draw[tick] (0.5,2.5) -- (0.3,2.5) node[left] {25};
  \draw[tick] (0.5,2.25) -- (0.4,2.25);
  \draw[tick] (0.5,3) -- (0.3,3) node[left] {30};
  \draw[tick] (0.5,2.75) -- (0.4,2.75);
  \node[tick] at (-0.1, -1.4) {\#$\mathrm{cliques}$};
  \node[tick] at (-0.1, -1.9) {\phantom{\#}$\mathrm{density}$};
  \node[tick] at (0.8,-1.4) {9};
  \node[tick] at (0.8,-1.9) {0.53};
  \node[tick] at (1.6,-1.4) {21};
  \node[tick] at (1.6,-1.9) {0.39};
  \node[tick] at (2.4000000000000004,-1.4) {44};
  \node[tick] at (2.4000000000000004,-1.9) {0.25};
  \node[tick] at (3.2,-1.4) {91};
  \node[tick] at (3.2,-1.9) {0.19};
  \node[tick] at (4,-1.4) {184};
  \node[tick] at (4,-1.9) {0.12};
  \node[tick] at (4.800000000000001,-1.4) {371};
  \node[tick] at (4.800000000000001,-1.9) {0.09};
  \node[tick] at (5.6000000000000005,-1.4) {740};
  \node[tick] at (5.6000000000000005,-1.9) {0.06};
  \node[tick, text width=7cm, align=center] at (3.5, -0.9) {Number of Vertices};
  \node[tick, text width=2cm, align=center] at (0.5, 4.4) {Average Time\\ in Minutes};
  \node[baseline, fill=ba.gray!50, inner sep=0.5ex] at (0, 5.2) {C};
  \node[baseline, anchor=west, inner sep=0.5ex] at (0.33, 5.175) {Tree Thickening, $k=\sqrt n$};
  \draw[timeout] (0.5, 3.5) node[left] {timeout} -- (5.6000000000000005, 3.5);
  \draw[axis] (0.5,0) to (6.1000000000000005, 0);
  \draw[axis] (0.5,0) to (0.5, 4);
  \draw[semithick, color=ba.pine] (0.8,0) -- (1.6,0);
  \draw[semithick, color=ba.pine] (1.6,0) -- (2.4000000000000004,0);
  \draw[semithick, color=ba.pine] (2.4000000000000004,0) -- (3.2,0);
  \draw[semithick, color=ba.pine] (3.2,0) -- (4,0);
  \draw[semithick, color=ba.pine] (4,0) -- (4.800000000000001,0);
  \draw[semithick, color=ba.pine] (4.800000000000001,0) -- (5.6000000000000005,0.0033333333333333335);
  \node[mean_dot, color=ba.pine] at (0.8,0) {};
  \draw[thin, color=ba.pine] (0.8,0) -- (0.8,0);
  \node[mean_dot, color=ba.pine] at (1.6,0) {};
  \draw[thin, color=ba.pine] (1.6,0) -- (1.6,0);
  \node[mean_dot, color=ba.pine] at (2.4000000000000004,0) {};
  \draw[thin, color=ba.pine] (2.4000000000000004,0) -- (2.4000000000000004,0.0016666666666666668);
  \node[mean_dot, color=ba.pine] at (3.2,0) {};
  \draw[thin, color=ba.pine] (3.2,0) -- (3.2,0.0016666666666666668);
  \node[mean_dot, color=ba.pine] at (4,0) {};
  \draw[thin, color=ba.pine] (4,0) -- (4,0.0016666666666666668);
  \node[mean_dot, color=ba.pine] at (4.800000000000001,0) {};
  \draw[thin, color=ba.pine] (4.800000000000001,0) -- (4.800000000000001,0.0016666666666666668);
  \node[mean_dot, color=ba.pine] at (5.6000000000000005,0.0033333333333333335) {};
  \draw[thin, color=ba.pine] (5.6000000000000005,0.0033333333333333335) -- (5.6000000000000005,0.005000000000000001);
  \node[tick, color=ba.pine] at (5.8500000000000005,0.20333333333333334) {cp};
  \draw[semithick, color=ba.violet] (0.8,0) -- (1.6,0);
  \draw[semithick, color=ba.violet] (1.6,0) -- (2.4000000000000004,0);
  \draw[semithick, color=ba.violet] (2.4000000000000004,0) -- (3.2,0);
  \draw[semithick, color=ba.violet] (3.2,0) -- (4,0.018333333333333333);
  \draw[semithick, color=ba.violet] (4,0.018333333333333333) -- (4.800000000000001,0.5900000000000001);
  \draw[semithick, color=ba.violet] (4.800000000000001,0.5900000000000001) -- (5.6000000000000005,3.5);
  \node[mean_dot, color=ba.violet] at (0.8,0) {};
  \draw[thin, color=ba.violet] (0.8,0) -- (0.8,0);
  \node[mean_dot, color=ba.violet] at (1.6,0) {};
  \draw[thin, color=ba.violet] (1.6,0) -- (1.6,0);
  \node[mean_dot, color=ba.violet] at (2.4000000000000004,0) {};
  \draw[thin, color=ba.violet] (2.4000000000000004,0) -- (2.4000000000000004,0.0016666666666666668);
  \node[mean_dot, color=ba.violet] at (3.2,0) {};
  \draw[thin, color=ba.violet] (3.2,0) -- (3.2,0.0033333333333333335);
  \node[mean_dot, color=ba.violet] at (4,0.018333333333333333) {};
  \draw[thin, color=ba.violet] (4,0.016666666666666666) -- (4,0.023333333333333334);
  \node[mean_dot, color=ba.violet] at (4.800000000000001,0.5900000000000001) {};
  \draw[thin, color=ba.violet] (4.800000000000001,0.4883333333333334) -- (4.800000000000001,0.73);
  \draw[thin, color=ba.violet] (4.700000000000001,0.4883333333333334) -- (4.9,0.4883333333333334);
  \draw[thin, color=ba.violet] (4.700000000000001,0.73) -- (4.9,0.73);
  \node[mean_dot, color=ba.violet] at (5.6000000000000005,3.5) {};
  \draw[thin, color=ba.violet] (5.6000000000000005,3.5) -- (5.6000000000000005,3.5);
  \node[tick, color=ba.violet] at (5.8500000000000005,3.8) {am};
  \draw[semithick, color=ba.yellow] (0.8,0) -- (1.6,0.016666666666666666);
  \draw[semithick, color=ba.yellow] (1.6,0.016666666666666666) -- (2.4000000000000004,3.5);
  \draw[semithick, color=ba.yellow] (2.4000000000000004,3.5) -- (3.2,3.5);
  \draw[semithick, color=ba.yellow] (3.2,3.5) -- (4,3.5);
  \draw[semithick, color=ba.yellow] (4,3.5) -- (4.800000000000001,3.5);
  \draw[semithick, color=ba.yellow] (4.800000000000001,3.5) -- (5.6000000000000005,3.5);
  \node[mean_dot, color=ba.yellow] at (0.8,0) {};
  \draw[thin, color=ba.yellow] (0.8,0) -- (0.8,0.0016666666666666668);
  \node[mean_dot, color=ba.yellow] at (1.6,0.016666666666666666) {};
  \draw[thin, color=ba.yellow] (1.6,0.0016666666666666668) -- (1.6,0.08833333333333333);
  \node[mean_dot, color=ba.yellow] at (2.4000000000000004,3.5) {};
  \draw[thin, color=ba.yellow] (2.4000000000000004,3.5) -- (2.4000000000000004,3.5);
  \node[mean_dot, color=ba.yellow] at (3.2,3.5) {};
  \draw[thin, color=ba.yellow] (3.2,3.5) -- (3.2,3.5);
  \node[mean_dot, color=ba.yellow] at (4,3.5) {};
  \draw[thin, color=ba.yellow] (4,3.5) -- (4,3.5);
  \node[mean_dot, color=ba.yellow] at (4.800000000000001,3.5) {};
  \draw[thin, color=ba.yellow] (4.800000000000001,3.5) -- (4.800000000000001,3.5);
  \node[mean_dot, color=ba.yellow] at (5.6000000000000005,3.5) {};
  \draw[thin, color=ba.yellow] (5.6000000000000005,3.5) -- (5.6000000000000005,3.5);
  \node[tick, color=ba.yellow] at (5.8500000000000005,3.2) {tw};
  \draw[semithick, color=ba.blue] (0.8,0) -- (1.6,0.0016666666666666668);
  \draw[semithick, color=ba.blue] (1.6,0.0016666666666666668) -- (2.4000000000000004,0.013333333333333334);
  \draw[semithick, color=ba.blue] (2.4000000000000004,0.013333333333333334) -- (3.2,2.4250000000000003);
  \draw[semithick, color=ba.blue] (3.2,2.4250000000000003) -- (4,3.5);
  \draw[semithick, color=ba.blue] (4,3.5) -- (4.800000000000001,3.5);
  \draw[semithick, color=ba.blue] (4.800000000000001,3.5) -- (5.6000000000000005,3.5);
  \node[mean_dot, color=ba.blue] at (0.8,0) {};
  \draw[thin, color=ba.blue] (0.8,0) -- (0.8,0.0016666666666666668);
  \node[mean_dot, color=ba.blue] at (1.6,0.0016666666666666668) {};
  \draw[thin, color=ba.blue] (1.6,0.0016666666666666668) -- (1.6,0.0033333333333333335);
  \node[mean_dot, color=ba.blue] at (2.4000000000000004,0.013333333333333334) {};
  \draw[thin, color=ba.blue] (2.4000000000000004,0.008333333333333333) -- (2.4000000000000004,0.03166666666666667);
  \node[mean_dot, color=ba.blue] at (3.2,2.4250000000000003) {};
  \draw[thin, color=ba.blue] (3.2,0.58) -- (3.2,3.5);
  \draw[thin, color=ba.blue] (3.1,0.58) -- (3.3000000000000003,0.58);
  \draw[thin, color=ba.blue] (3.1,3.5) -- (3.3000000000000003,3.5);
  \node[mean_dot, color=ba.blue] at (4,3.5) {};
  \draw[thin, color=ba.blue] (4,3.5) -- (4,3.5);
  \node[mean_dot, color=ba.blue] at (4.800000000000001,3.5) {};
  \draw[thin, color=ba.blue] (4.800000000000001,3.5) -- (4.800000000000001,3.5);
  \node[mean_dot, color=ba.blue] at (5.6000000000000005,3.5) {};
  \draw[thin, color=ba.blue] (5.6000000000000005,3.5) -- (5.6000000000000005,3.5);
  \node[tick, color=ba.blue] at (5.8500000000000005,3.5) {li};
\end{tikzpicture}
   \caption{The same plots as in Fig.~\ref{figure:experiments} for the random graphs
     from Section~\ref{section:thickening}.}
   \label{figure:experiments2}
 \end{figure*}
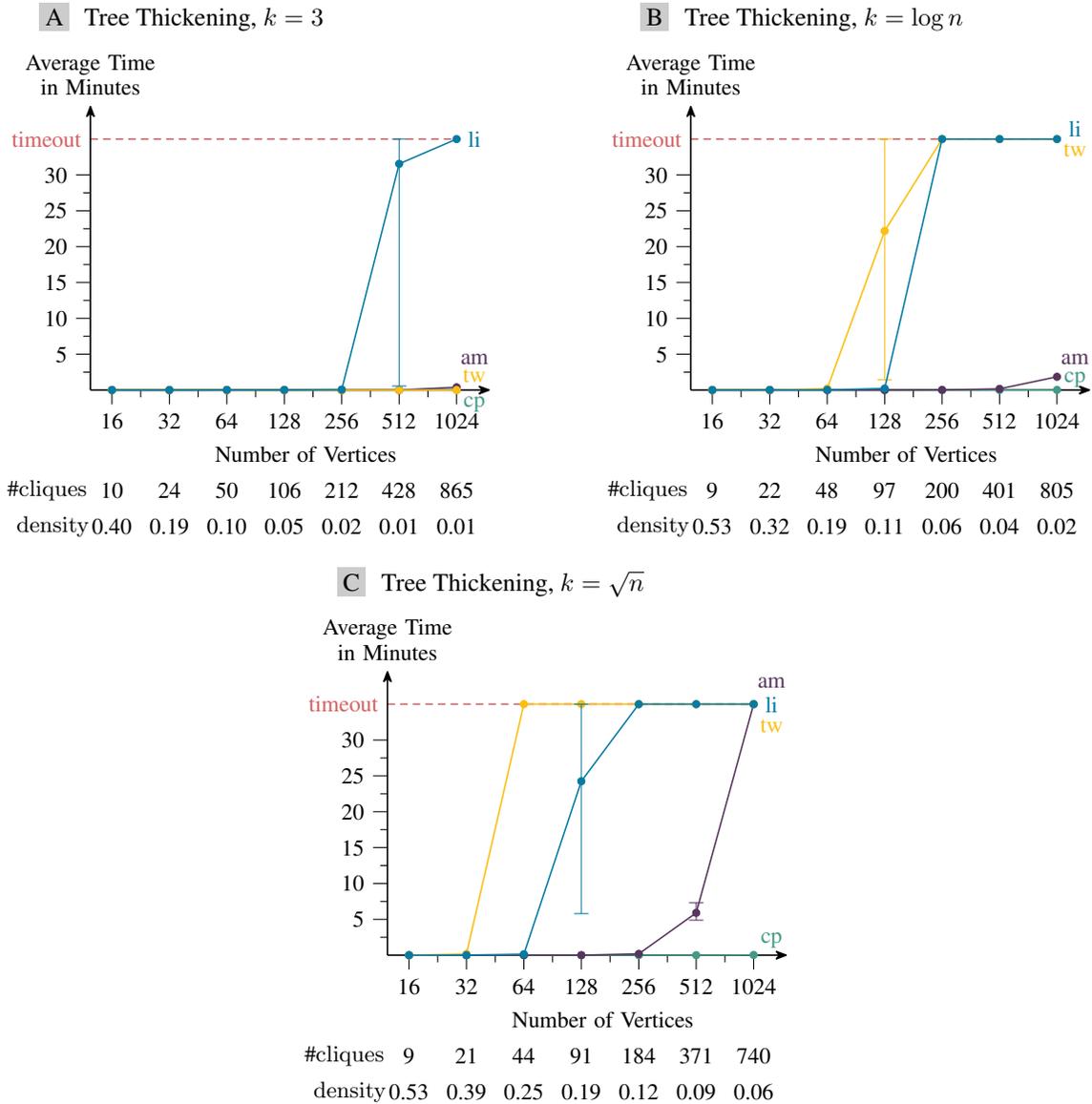
\clearpage 
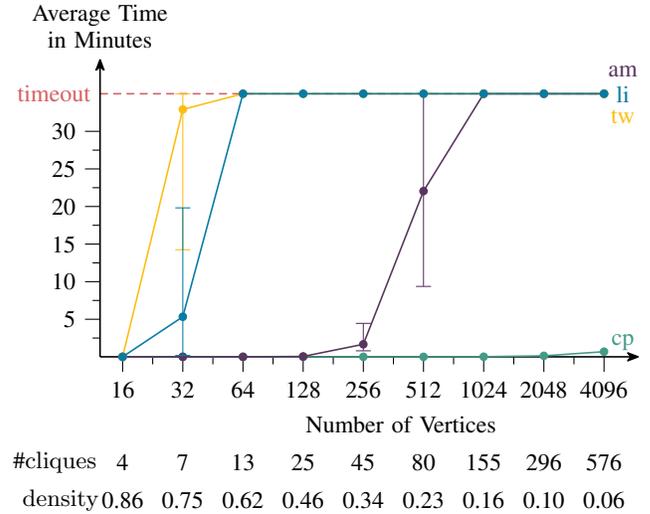
\end{document}